\documentclass[11pt,oneside,onecolumn,DIV=14,headings=normal]{scrartcl}

\usepackage[T1]{fontenc}
\usepackage[utf8]{inputenc}
\usepackage{amsmath,amssymb,amsthm,bm}
\usepackage{graphicx,booktabs,longtable,array,calc}
\usepackage{enumitem}
\usepackage{microtype}
\usepackage[numbers]{natbib}
\usepackage{xurl}
\usepackage[colorlinks=true,linktoc=all]{hyperref}
\hypersetup{colorlinks=true,linkcolor=blue,citecolor=blue,urlcolor=blue,
  pdftitle={Identifiability and Observability of Deep Normalized Attention}}
\newtheorem{theorem}{Theorem}
\newtheorem{proposition}{Proposition}[section]
\newtheorem{corollary}[proposition]{Corollary}

\DeclareMathOperator{\sym}{sym}

\usepackage{xcolor}

\title{Identifiability and Observability of Deep Normalized Attention}
\author{Pranav Venkata Konda\thanks{Columbia University}}
\date{}

\makeatletter
\let\ArxivBeginTheorem\@begintheorem
\def\@begintheorem#1#2[#3]{\upshape\ArxivBeginTheorem{#1}{#2}[#3]}
\makeatother
\renewcommand{\appendix}{%
  \setcounter{section}{0}%
  \renewcommand{\thesection}{\Alph{section}}%
  \renewcommand{\theHsection}{appendix.\Alph{section}}%
}
\hypersetup{colorlinks=true,allcolors=red,pdfborder={0 0 0}}
\AtBeginDocument{%
  \hypersetup{allcolors=red}%
  \pagestyle{plain}%
}

\begin{document}
\maketitle

\begin{abstract}
We study which parameters of deep, unmasked, single-head
attention are determined by its input--output function.
For known positive nonconstant real-analytic normalizers,
the function generically determines the effective scores
and combined value map up to the signs induced by even
normalizers. This proves the real-analytic case of a
conjecture of  
Henry--Marchetti--Kohn, including softmax.
We then classify exceptional fibers under explicit
normalizer conditions, identifying when collapse makes
later scores unobservable, and establish sharp Taylor
orders for local identification. Near simultaneous
query/key collapse, we compute the complete native
Jacobian decay spectrum on separating finite input banks.
For common first nonconstant normalizer degree $k$,
layer $i$ has contact order $2k3^{i-1}-1$, with exact
multiplicities and kernel dimension. High-precision and
automatic differentiation calculations illustrate the
resulting loss of numerical sensitivity.
\end{abstract}
\tableofcontents
\section{Introduction}\label{sec:introduction}

A natural question to ask in deep learning is what
information about a neural network is uniquely determined
by its input--output function. For attention, query/key
factorizations and changes of coordinates between layers
allow different weights to represent the same function
\citep{henry2025}. Thus we pass to \emph{effective}
parameters: the combined value map and layerwise score
matrices in the original input coordinates. Product-level
descriptions are also central to Transformer circuit
analysis \citep{elhage2021}, so identifying these quantities
gives a precise target for recovering internal parameters
from the function.

This question matters when behavioral tests support a
safety argument about a model's internal computation.
Such an argument requires observations that distinguish
parameters for which the claimed property differs. We
therefore ask which parameter distinctions survive in the
function, and how sensitively finite observations detect
them. Pure serial attention isolates how repeated token
mixing affects both questions.

Henry, Marchetti, and Kohn conjectured generic effective
identification for positive injective normalizers
\citep[Conjecture~3.10]{henry2025}. Our main result establishes
the real-analytic case, including softmax. More generally,
Theorem~\ref{thm-identification} determines the complete
effective fiber whenever the combined value map is nonzero
and every score has a nonzero symmetric part. The only
ambiguities are the score signs induced by even normalizers.
In particular, competing representations may be degenerate,
the scores may be rank deficient, and the value map may
have a nontrivial kernel.

We then ask how identification breaks down. A layer can
send every input to \emph{consensus}, where all token vectors
agree, making later scores unobservable.
Theorem~\ref{thm-exceptions} classifies these exceptional
fibers under explicit normalizer conditions. Near collapse,
identifiable directions can still have very weak effects
on observed outputs. We therefore study finite observations:
Theorem~\ref{thm-jets} gives sharp Taylor orders for local
identification, and Theorem~\ref{thm-spectrum} computes the
vanishing orders of all positive native Jacobian singular
values on separating input banks near simultaneous
query/key collapse, together with their multiplicities
and the exact kernel dimension.

The argument follows the token means and contrasts, where
contrast refers to the centered token vectors. Near
consensus, a layer changes the mean at second order in the
incoming contrast while transmitting contrast at third
order. Thus deeper scores appear at successively higher
orders. Subtracting the exact mean of a recovered prefix
isolates the next score. This connects identification of
the parameters to the loss of numerical sensitivity under
repeated contrast contraction.

\section{Effective parameters and identification}\label{sec:model}
We fix a depth $D \geq 1$, a token count $n \geq 2$, and input
$X = (x_1, \ldots, x_n)$, with input width $d$ and output
width $p$. We consider known normalizers $\phi_i$ that are
positive, nonconstant, and real analytic near zero.
Multiplying $\phi_i$ by a positive constant leaves the
normalized attention weights unchanged, so without loss of
generality we set $\phi_i(0) = 1$. For a score matrix $M \in
\mathbb R^{d\times d}$,
we define the attention layer with identity values as 

\begin{equation}\label{eq:attention-layer}
S^\phi_M(X)_a = \frac{\sum_{b=1}^n \phi(x_a^\top M
x_b)x_b}{\sum_{b=1}^n \phi(x^\top_aMx_b)},
\end{equation}
where the $a$-th output token is a weighted average of input
token vectors. The depth-$D$ network in effective parameters
is defined as 
\begin{equation}\label{eq:effective-network}
F_{L, \mathbf M} = LS^{\phi_D}_{M_D}\circ \cdots \circ
S^{\phi_1}_{M_1},
\end{equation} 
where $\mathbf M = (M_1, \ldots, M_D)$ are the effective
score matrices, and $L$ is the combined value map. Setting
$\phi(s) = e^s$ gives softmax, with the usual scaling of
dot-product scores incorporated into the matrices $\mathbf
M$ \citep{vaswani2017}.

To relate this model to native weights, take query, key,
and value matrices $Q_i$, $K_i$, and $V_i$ of compatible sizes,
let $A_i = Q_iK_i^\top$, $P_0 = I$, and $P_i = V_i\cdots
V_1$. These weights determine the effective tuple by 
\begin{equation}\label{eq:reduction}
M_i = P_{i-1}^\top A_i P_{i-1} \text{, } L = P_D.
\end{equation}

Here, $M_i$ is the layer's score bilinear form in the input
feature coordinates, and $L$ combines all of the value maps.
This reduction holds for singular or rectangular values. 

The expression $A_i = Q_iK_i^\top$ is a factorization of the
score matrix into query and key matrices. Suppose these
factors have $r_i$ columns. Then every invertible $G_i \in
\operatorname{GL}_{r_i}(\mathbb R)$ gives another
factorization with the same score: 
\begin{equation}\label{eq:factor-gauge}
\begin{aligned}
 (Q_i,K_i)&\longmapsto(Q_iG_i,K_iG_i^{-\top}),\\
 (Q_iG_i)(K_iG_i^{-\top})^\top&=Q_iK_i^\top.
\end{aligned}
\end{equation} These
  transformations are \emph{factor gauges}, and they
  preserve the effective tuple and thus the realized
  function. 

Write the native weights as $\theta = ((Q_i, K_i,
V_i))^D_{i=1}$ and their effective combinations as $\eta =
(L, M_1, \ldots, M_D) = (L, \mathbf M)$. Then Equation~\eqref{eq:reduction} defines the parameter map $\pi: \theta \mapsto
\eta$, which has entries polynomial in the native weights. 

Write $F_\eta=F_{L,\mathbf M}$ for the input--output
function in Equation~\eqref{eq:effective-network}, and
$[F_\eta]_0$ for its analytic germ at the zero input.
The \emph{realization map} sends $\eta$ to this germ,
giving the factorization
\begin{equation}\label{eq:realization-factorization}
\theta \xrightarrow{\pi} \eta = (L, \mathbf M)
\xrightarrow{\mathcal R} [F_\eta]_0.
\end{equation}

For fixed dimensions and normalizers we define the effective
fiber $\mathcal R^{-1}([F]_0)$ as the set of all effective-parameter tuples
realizing the germ $[F]_0$: 
\begin{equation}\label{eq:effective-fiber}
\mathcal R^{-1}([F]_0) = \{(L, \mathbf M): [F_{L, \mathbf
    M}]_0=[F]_0
  \}.
\end{equation} 

Equality here means equality of analytic germs
(i.e. the input functions agree on some neighborhood of $X=0$),
and it extends to the whole connected input domain when 
normalizers are globally positive analytic.

Within this model, a property $P(\eta)$ can be inferred
from the realized function only if it is constant on each
fiber:
\begin{equation}\label{eq:behavioral-property}
\mathcal R(\eta)=\mathcal R(\widetilde\eta)
\ \Longrightarrow\ P(\eta)=P(\widetilde\eta).
\end{equation}
Thus a parameter claim must first be determined by the
function before we ask whether finite observations can
resolve it.

Theorem~\ref{thm-identification} determines the entire effective fiber when one
representative satisfies its hypotheses:
\begin{theorem}[Effective identifiability]\label{thm-identification} Under the model
  assumptions above, let $\eta = (L, \mathbf M)$ satisfy
  $L\neq0$, and $\sym M_i = \frac{M_i + M_i^\top}{2} \neq 0$
  for every layer. Then the fiber $\mathcal R^{-1}(\mathcal
  R(\eta))$ consists exactly of the tuples 
  $\widetilde \eta = (\widetilde L, \widetilde{\mathbf M})$ that satisfy
  \begin{equation}\label{eq:identifiable-fiber}
\widetilde L = L, \text{ } \widetilde{M_i} \in
    \begin{cases} \{M_i\} &\text{ if } \phi_i \text{ is not
      even} \\ \{M_i, -M_i\} &\text{ if } \phi_i \text{ is
    even}. \end{cases}
\end{equation} Sign choices are independent
    across layers. If $\phi'_D(0) \neq 0$, then the
    condition $\sym{M_D} \neq 0$ is unnecessary, as the
    final score matrix may be arbitrary. 
\end{theorem}

Theorem~\ref{thm-identification} is generic in both the effective and native
parameters. For every native architecture with positive
widths, its hypotheses exclude only a proper algebraic
subset (we prove this in Appendix~\ref{appendix-a}). 

An injective normalizer cannot be even, and thus Theorem~\ref{thm-identification} gives singleton effective fibers for the analytic subclass
of the Henry--Marchetti--Kohn conjecture. Injectivity is sufficient but not necessary.
The remaining freedom in native factorizations introduces
no additional ambiguity in the effective tuple. 

\subsection{Effective fiber determination}\label{sec:layer-isolation}
We first recover the value map on consensus inputs, then
recover the scores from perturbations of consensus.
Suppose $\eta$ and $\widetilde\eta$ realize the same
function germ. Since every effective layer fixes consensus, \begin{equation}\label{eq:consensus-readout}
F_\eta(m, \ldots, m) = (Lm, \ldots, Lm) 
  = F_{\widetilde\eta}(m,
\ldots, m),
\end{equation} we obtain $L = \widetilde L$. To
recover the score matrices, we then perturb the input
vectors from consensus and measure how each layer changes
the token mean. We first consider one softmax layer with two
vector tokens $x_{\pm} = m\pm \varepsilon z$. Then $\phi(s)
= e^s$, and we write their outputs as $y_{\pm} =
S^\phi_M(X)_\pm$. We expand in the spread parameter
$\varepsilon$ to obtain the mean 
\begin{equation}\label{eq:softmax-mean}
\frac{y_+ + y_-}{2} = m + \varepsilon^2z(m^\top Mz) +
  O(\varepsilon^4),
\end{equation} and the contrast
  \begin{equation}\label{eq:softmax-contrast}
\frac{y_+ - y_-}{2} = \varepsilon^3z(z^\top Mz) +
    O(\varepsilon^5).
\end{equation} Score enters the mean at second
    order, but the outgoing contrast begins at third order.
    When $\sym{M} \neq 0$, $z^\top Mz$ is not identically zero as $z$ varies, so a
    subsequent layer receives a smaller contrast and thus
    makes its first contribution to the mean at a higher
    order. We can then recover the scores successively due
    to this separation of orders, and this holds in general
    for the analytic normalizers in Theorem~\ref{thm-identification} and any token
    count. 

    Consider the line probes $x_a = m + \varepsilon c_az$,
    where 
    \begin{equation}\label{eq:line-probe-normalization}
\frac{1}{n}\sum_a c_a = 0, \text{ and } \frac{1}{n}
    \sum_a c_a^2 = 1.
\end{equation} The affine line containing these
    probes is preserved by every effective layer. For a
    layer with normalizer $\phi$, we write $h = \log{\phi}$,
    and $g = h'$. We define 
    \begin{equation}\label{eq:mean-signature}
H_M(m, z) = g(m^\top M m)m^\top Mz,
\end{equation} which is the
    layer's leading contribution to the mean. 

    In particular,
    if the incoming mean tends to $m$, and the centered
    tokens have leading terms of the form
    $\varepsilon^r\alpha c_az$, then the mean increment and
    outgoing centered tokens have leading terms of the form 
    \begin{equation}\label{eq:leading-line-terms}
\varepsilon^{2r}\alpha^2 zH_M(m, z), \text{ and }
      \varepsilon^{3r}\alpha^3c_azK_M(m, z),
\end{equation}
      respectively. Let $y_a$ be the output tokens and $\bar
      x, \bar y$ be the input and output means,
      respectively. The one-layer expansion is 
      \begin{equation}\label{eq:contrast-order}
\begin{aligned}
        \bar y - \bar x &= \varepsilon^{2r}\alpha^2z H_M(m,
        z) + O(\varepsilon^{2r+1}), \\ y_a - \bar y &=
        \varepsilon^{3r}\alpha^3c_a zK_M(m, z) +
        O(\varepsilon^{3r+1}). 
      \end{aligned}
\end{equation}

      We compute $K_M$ and show it is also a
      nonzero analytic germ whenever $\sym{M} \neq 0$ in
      Appendix~\ref{appendix-a}. 

    Let $X^{(j)} = S_{M_j}^{\phi_j}(X^{(j-1)})$, with
    $X^{(0)} = X$. We define the prefix mean and centered
    tokens by
    \begin{equation}\label{eq:prefix-statistics}
\mu_j = \frac{1}{n}\sum_a x_a^{(j)}, \delta_a^{(j)} =
      x_a^{(j)} - \mu_j.
\end{equation} Under the hypotheses of Theorem~\ref{thm-identification}, we obtain (from Appendix~\ref{appendix-a})
      \begin{equation}\label{eq:prefix-contrast}
\delta_a^{(j)} = \varepsilon^{3^j}a_jc_az + \cdots,
\end{equation} where $a_0 = 1$, $a_j = a_{j-1}^3 K_{M_j} \neq 0$
    as analytic germs.
    
    If the first $i-1$ layer maps are known, their exact
    mean $\mu_{i-1}$ and amplitude $a_{i-1}$ are also known.
    Suppose $\overline{F_\eta} = L\mu_D$ is the 
    observed output mean. It satisfies the telescoping
    identity 
    \begin{equation}\label{eq:telescoping-mean}
\overline{F_\eta} - L\mu_{i-1} = L\sum_{j=i}^D (\mu_j
      - \mu_{j-1}),
\end{equation} where Equation~\eqref{eq:contrast-order} gives spread
    order at least $2\cdot3^{j-1}$ for the $j$th increment.
    Thus only layer $i$ contributes at the first remaining
    order. Writing $[\varepsilon^k]$ for coefficient
    extraction, Appendix~\ref{appendix-a} gives
    \begin{equation}\label{eq:isolated-signature}
[\varepsilon^{2\cdot 3^{i-1}}](\overline{F_\eta} -
      L\mu_{i-1}) = a_{i-1}^2(Lz)H_{M_i}(m, z).
\end{equation} The
      exact subtraction removes all prefix terms (including
      the higher-order terms), and coefficient extraction
      removes the suffix. 

    Now, let $\mathcal R(\widetilde\eta) = \mathcal
    R(\eta)$, and suppose the recovered prefix maps agree.
    Choose a nonzero row $\ell$ of $L$: subtracting the
    coefficient identities gives 
    \begin{equation}\label{eq:analytic-cancellation}
a_{i-1}^2(\ell z)(H_{M_i} - H_{\widetilde{M_i}}) = 0.
\end{equation} Since $\mathbb R\{m,z\}$ is an integral domain,
    $a_{i-1}^2(\ell z) \neq 0$ implies that $H_{M_i} =
    H_{\widetilde{M_i}}$. While the factor may vanish at
    select individual probes, it is nonzero as a germ,
    and this permits cancellation.
    Note from Equation~\eqref{eq:contrast-order} we also
    have 
    \begin{equation}\label{eq:competing-contrast-order}
\operatorname{ord}_\varepsilon{\widetilde\delta_a^{(i)}}\geq
    3^i,
\end{equation} even when the other layer's coefficient
    $K_{\widetilde M_i}$ vanishes identically. Its suffix
    therefore contributes only at higher spread orders and
    thus we make no regularity assumption on
    $\widetilde\eta$. 

    If $g(0)\ne0$, the term $g(0)m^\top Mz$ linear
    in $m$ determines every $M$. Otherwise, setting $z=m$
    gives 
    \begin{equation}\label{eq:diagonal-signature}
H_M(m, m) = g(m^\top Mm)m^\top Mm.
\end{equation} Its lowest nonzero homogeneous term determines
    $m^\top Mm$ up to one global sign. The full signature
    determines the remaining bilinear form and rules out
    the negative sign unless $\phi$ is even. In that case,
    $S^\phi_M=S^\phi_{-M}$, so either sign gives the same
    recovered layer map. This argument also establishes
    the nonzero symmetric part of each competing score
    for which the theorem requires it. 

    Thus depth separates successive scores by their order of
    appearance in the mean, but makes their recovery
    increasingly sensitive to small coefficients. The proof
    establishes exact identification by analytic
    cancellation, but does not provide stable numerical
    division. Additionally, smooth normalizers can hide
    score dependence from every zero-input Taylor
    coefficient (even if injective), and we provide an
    example of this in Appendix~\ref{sec:flat-boundary}. 

    \section{Suffix erasure by a layer}\label{sec:collapse}

    Let $\mathcal C = \{m\mathbf 1^\top: m \in \mathbb
    R^d\}$ be the \emph{consensus subspace}, which every
    effective layer fixes pointwise. Write the token mean
    as $\mu(X)=\frac1nX\mathbf1$ and the centered token
    matrix as $\Delta(X)=X-\mu(X)\mathbf1^\top$. 

    For each
    prescribed normalizer $\phi_i$, we define its collapse
    locus by 
    \begin{equation}\label{eq:collapse-locus}
Z_i = \{ M : \Delta \circ S^{\phi_i}_M = 0 \},
\end{equation}
    such that $M_i \in Z_i$ implies that the entire image of
    the layer lies in $\mathcal C$ (as an identity of input
    germs). If $M_s \in Z_s$, then $X^{(s)} \in \mathcal C$,
    which implies that $F_\eta(X) = LX^{(s)}$, and thus
    changing $M_{s+1}, \ldots, M_D$ does not change the
    realized function. We now identify the collapse loci,
    and show that the free suffix $M_{s+1}, \ldots, M_D$
    accounts for the entire effective fiber. For each layer,
    we write $h_i = \log \phi_i$, with
    normalization such that $\phi_i(0) = 1, h_i(0) = 0$. 

    \begin{theorem}[Effective fibers at the first collapse]\label{thm-exceptions}
      Under the model assumptions of Section~\ref{sec:model},
      suppose $h_i'(0) \neq 0$ for all $i = 1, \ldots, D$.
      Then we have: 
      \begin{enumerate}[label=\roman*.,ref=\roman*]
        \item\label{thm-exceptions:curvature} For $n \geq 2$ arbitrary,
          suppose also that $h_i''(0)
          \neq 0$. Then $Z_i = \{0\}$. 
        \item\label{thm-exceptions:two-tokens} For $n=2$ and general $h_i''(0)$, we have 
          \begin{equation}\label{eq:two-token-collapse-locus}
Z_i = \begin{cases} \left \{ M : M^\top = -M
            \right \} &\text{ if } h_i \text{ is odd,} \\ 
          \{0\} &\text{ else,} \end{cases}
\end{equation} where oddness
          refers to the germ of $h_i$ at 0. 
      \end{enumerate}
      In either regime above, let $\eta=(L,M_1,\ldots,M_D)$,
      with $L\neq0$. Let $s$ be the minimal index such that
      $M_s \in Z_s$, and $s = D$ if no such index exists.
      Then 
      \begin{equation}\label{eq:first-collapse-fiber}
\begin{aligned}
 &\mathcal R^{-1}(\mathcal R(\eta))\\
 &\quad=\{(L,M_1,\ldots,M_s,N_{s+1},\ldots,N_D):\\
 &\hspace{3em}N_j\in\mathbb R^{d\times d}\ (j>s)\}.
\end{aligned}
\end{equation}
      In particular
      the fiber is a singleton if no layer collapses before
      the last. For $L=0$, the fiber is 
      \begin{equation}\label{eq:zero-readout-fiber}
\mathcal R^{-1}(\mathcal R(0,\mathbf M))
          =
            \{(0,\mathbf N):
          \mathbf N\in(\mathbb R^{d\times d})^D\}.
\end{equation} 
    \end{theorem}

    Note that in collapse the score $M_s$ itself remains
    identifiable, and can thus affect the location of the
    consensus it produces: for example, consider two-token
    softmax, with inputs $m\pm z$ and a skew score $M$. The
    output tokens both satisfy 
    \begin{equation}\label{eq:skew-consensus}
y_+ = y_- = m + z \tanh(m^\top Mz),
\end{equation} but their
    common value still depends on $M$. Let us additionally
    consider an example where $D = 3$, $L\neq 0$, and $\sym
    M_1 \neq 0$, $M_2^\top = -M_2$. Then 
    \begin{equation}\label{eq:depth-three-fiber}
\mathcal R^{-1}(\mathcal R(\eta)) = \{(L, M_1, M_2,
    N) : N \in \mathbb R^{d\times d}\}.
\end{equation}
    Here, an input--output audit cannot determine even
    whether $N=0$. Since every choice of $N$ gives the same
    function, additional inputs and higher precision cannot
    resolve this ambiguity. Testing a property that varies
    with $N$ requires observations beyond the realized
    function. 

    The assumption of nonzero slope is essential, even in
    the case of two tokens. For example, consider $\phi_1(s)
    = e^{s^2}$, a nonzero skew $M_1$, and $M_2 = 0$. We
    have $h_1'(0) = 0$, and 
    \begin{equation}\label{eq:zero-slope-example}
\begin{aligned}
        \mu(S_{M_1}^{\phi_1}(X)) &= \mu(X), \\
        \Delta(S_{M_1}^{\phi_1}(X)) & \not\equiv 0, \\ 
        S_0^{\phi_2} S_{M_1}^{\phi_1}(X) &= \mu(X)\mathbf
        1^\top. 
      \end{aligned}
\end{equation}

    The first layer preserves contrast, but the second
    layer's averaging erases its score dependence. Thus
    the function need not identify the prefix
    through the first collapse when the slope condition
    fails. We prove the classification and this
    counterexample in Appendix~\ref{appendix-b}. 

    Note that softmax satisfies $h_i'(0) = 1$, but its
    skew-score collapse depends on the token geometry: with
    three or more tokens, a nonzero skew score can preserve
    contrast on a noncollinear cloud. Theorem~\ref{thm-exceptions}(\ref{thm-exceptions:curvature}) does not cover this case, as $h''$
    vanishes at 0. We use multidirectional probes to recover
    a broader class of prefixes in Section~\ref{sec:jets} to obtain
    a complete depth-two softmax classification. The general
    case of singular prefixes mixing skew and nonskew
    scores remains open for deeper stacks. 

    \section{Parameter identification by finite input jets}\label{sec:jets}

    Let $j_0^{N}$ denote the \emph{Taylor jet} operator,
    which retains the input derivatives at $X = 0$ through
    total order $N$. We then define the \emph{finite Taylor
    observation map} $T_{N}$ as applying $j_0^{N}$ to the
    function realized by an effective parameter tuple: 
    \begin{equation}\label{eq:taylor-observation}
T_N(\eta) = j_0^N F_\eta = \left (
        \frac{\partial_X^\alpha F_\eta(0)}{\alpha !} \right
        )_{|\alpha| \leq N}.
\end{equation}
    Here $\alpha$ is a multi-index over the input entries.
    Since truncation retains only part of the function, 
    \begin{equation}\label{eq:jet-fiber-inclusion}
\mathcal R^{-1}(\mathcal R(\eta)) \subseteq
      T_N^{-1}(T_N(\eta)).
\end{equation} 

    \begin{theorem}[Sharp finite Taylor identification]\label{thm-jets}
      Under the model assumptions of Section~\ref{sec:model},
      suppose the normalizers have a common first
      nonconstant Taylor degree $k$, namely 
      \begin{equation}\label{eq:normalizer-leading-term}
\log{\phi_i(s)} = b_is^k + O(s^{k+1}),
\end{equation} where
      $b_i \neq 0$. Let $\eta = (L, \mathbf M)$ be such that
      $L \neq 0$ and $\sym M_i \neq 0$ for all $i$. We set
      $N_D = 2k3^{D - 1} + 1$. 

      The Taylor observation map $T_{N_D}$ locally
      identifies $\eta$, as its differential is injective
      and it has a local analytic inverse onto its image.
      For all $N < N_D$, $T_N$ is independent of the final
      score $M_D$, so the above degree is sharp. 

      For $k=1$ the final matrix may be arbitrary, in which
      case we have $T_{N_D}^{-1}(T_{N_D}(\eta)) = \{\eta\}$,
      with $T_{N_D}$ having a rational local inverse on
      suitable charts. 
    \end{theorem}
    For softmax, the sharp degrees are 3, 7, 19, and 55 at
    depths one through four. Consequently, any procedure using only zero-input
    derivatives of lower degree cannot distinguish changes
    in the final effective score. When the first nonconstant degree is even, a jet can
    give local coordinates while failing to distinguish
    distant sign branches. We discuss the sharp sign-resolving degree,
    heterogeneous leading degrees, and adaptive degrees on
    two-token exceptional strata in Appendix~\ref{appendix-c}. 

    Note that both $T_N$ and $T_N \circ \pi$ are polynomial.
    Normalization denominators are of the form $n + u(X)$
    where $u(0) = 0$, with reciprocal in degree $N$ as 
    \begin{equation}\label{eq:reciprocal-truncation}
(n+u)^{-1} \equiv \frac{1}{n} \sum_{j=0}^N \left
    (-\frac{u}{n} \right )^j \mod{(X)^{N+1}}.
\end{equation} As only
    fixed powers of $n$ appear as denominators, we see that
    the local coordinates are polynomial observations of
    $\mathcal R$. Note the full softmax function and
    ordinary input evaluations remain analytic and generally
    are nonpolynomial. 

    We can also move beyond collinear probes: namely, the
    symmetric-part condition is specific to the collinear
    case, and is not a universal barrier to identification
    more broadly. As an example, suppose $d \geq 2$, $n \geq
    3$, all $h_i'(0) \neq 0$, $L\neq 0$, and each prefix
    $M_i$ for $i < D$ satisfies $\det(U^\top M_i U)
    \not\equiv 0$ as a polynomial in $U \in \mathbb
    R^{d\times 2}$. Then the jet of degree $2\cdot 3^{D-1} +
    1$ identifies the complete effective tuple (with
    arbitrary final score), and is also an immersion. Note
    this includes every nonzero skew prefix in any
    dimension. 

    Now suppose $U \in \mathbb R^{d\times 2}$ has rank two.
    We restrict the input tokens to the plane
    $\operatorname{im}(U)$ by writing $X = UZ$. The
    effective parameters of this restriction are 
    \begin{equation}\label{eq:restricted-tuple}
\eta_U = \bigl (LU, (U^\top M_i U)_{i=1}^D\bigr ),
\end{equation}
    so $F_\eta(UZ) = F_{\eta_{U}}(Z)$.

    The determinant hypothesis along with $L \neq 0$
    guarantees a nonempty open set of rank-two matrices $U$ such
    that 
    \begin{equation}\label{eq:good-two-planes}
LU \neq 0, \qquad \det(U^\top M_i U) \neq 0 \quad (i
      < D).
\end{equation} Suppose $\widetilde \eta$ realizes the same
      function as $\eta$: restricting both networks to
      inputs $X = UZ$ gives $F_{\eta_U} =
      F_{\widetilde{\eta_U}}$. We show that varying the
      token covariance identifies a network with
      two-dimensional inputs, nonzero readout, and
      invertible prefix scores in Appendix~\ref{appendix-c}. From
      that we obtain 
      \begin{equation}\label{eq:restricted-identities}
\begin{gathered}
 (L-\widetilde L)U=0,\qquad U^\top(M_i-\widetilde M_i)U=0\\
 (i=1,\ldots,D).
\end{gathered}
\end{equation}
        Note each matrix entry is polynomial in $U$, and
        since these polynomials vanish on a nonempty open
        set, they vanish for all $U$. Then writing $U =
        (u,v)$ we get 
        \begin{equation}\label{eq:ambient-identities}
(L - \widetilde L)u = 0, \qquad u^\top(M_i -
          \widetilde M_i)v = 0
\end{equation} for all $u, v$, which
          implies $\eta=\widetilde\eta$. The argument also
          applies to degree-$N_D$ jets: linear substitution
          preserves their equality, and covariance recovery
          uses only these degrees. Appendix~\ref{appendix-c}
          proves this and the following corollary. 
        \begin{corollary}[Complete depth-two fibers under
          nonzero slope]\label{cor-depth-two}
          Let $D = 2$, $n \geq 3$, and $h_1'(0)h_2'(0) \neq
          0$. Then we have: 
          \begin{enumerate}[label=\roman*.,ref=\roman*]
            \item\label{cor-depth-two:nonzero-prefix} If $L\neq0$ and $M_1\neq0$, the degree-seven
          input jet determines the complete effective tuple,
          and the degree-seven Taylor map is an immersion.
          No lower jet depends on $M_2$. 
            \item\label{cor-depth-two:zero-prefix} If $L \neq 0$ and $M_1 = 0$, then the
              complete fiber consists of exactly the same
              $L$, $M_1$, and an arbitrary second score
              $M_2$, and is determined by degree 3. 
            \item\label{cor-depth-two:zero-readout} If $L = 0$, the fiber is the zero readout
              and every score tuple. 
          \end{enumerate}
        \end{corollary}

    A nonzero first score is either nonskew and thus
    covered by Theorem~\ref{thm-identification} and Theorem~\ref{thm-jets}, or skew and
    covered by the two-plane result. With Theorem~\ref{thm-exceptions},
    this classifies depth-two softmax fibers for all $n \geq
    2$. Deeper mixed singular prefixes remain open. 

    \subsection{Native fibers and serial composition}\label{sec:native-fibers}

    We connect effective identification to native redundancy
    with the reduction map $\pi$. For any realized function
    germ $f$, we have 
    \begin{equation}\label{eq:native-fiber}
(\mathcal R \circ \pi)^{-1}(f) = \pi^{-1}(\mathcal
      R^{-1}(f)).
\end{equation} Wherever $T_N$ gives local effective
      coordinates, the local native fiber is thus a fiber of
      $\pi$. 

      \begin{proposition}[Local multiplication geometry]\label{prop-native-geometry} Let
        $Q_i, K_i \in \mathbb R^{d\times r_i}$, and let $V_i
        \in \operatorname{GL}_d(\mathbb R)$. Suppose $\Psi(\eta)$ is a
        finite Taylor or ordinary-input observation with
        injective differential at the marked effective
        tuple. Then $\Psi \circ \pi$ is locally analytically
        equivalent to 
        \begin{equation}\label{eq:native-normal-form}
(\mathbf P, L, \widetilde{\mathbf Q},
          \widetilde{\mathbf K}) \mapsto (L,
          (\widetilde{Q_i}\widetilde{K_i}^\top)_{i=1}^D, 0),
\end{equation} where $\mathbf P = (P_1, \ldots, P_{D-1})$. Its
        fiber germ is the product of the marked
        matrix-multiplication fibers and $(D-1)d^2$ smooth
        value-coordinate directions. 
      \end{proposition}

      Here, we have source coordinates $\widetilde{Q_i} =
      P^\top_{i-1}Q_i, \widetilde{K_i} = P_{i-1}^\top K_i$,
      and inverse $V_i = P_iP_{i-1}^{-1}$. At
      full-column-rank factor pairs, matrix-product geometry
      \citep{pepinlehalleur2024} gives 
      \begin{equation}\label{eq:native-rank}
\operatorname{rank} D(\Psi \circ \pi) = d^2 +
      \sum_i r_i(2d - r_i).
\end{equation} The remaining directions are
      native gauges. We provide the proof and limits at
      collapse of this proposition in Appendix~\ref{appendix-geometry}. 

      If we have a marked tuple with $L_* \neq 0$,
      nonzero logarithmic slopes, nonskew scores before a
      layer $s$, and $M_s^* = 0$, the effective fiber enlarges. For a suitable finite
      observation $\Psi$, the residual ideal is generated
      by the listed matrix entries: 
      \begin{equation}\label{eq:first-collapse-ideal}
\begin{aligned}
          I_{\eta_*}(\Psi) &= (\Psi - \Psi(\eta_*)) \\ 
                           &= (L - L_*, (M_i - M_i^*)_{i<s},
                           M_s). 
        \end{aligned}
\end{equation} 

      For width-one softmax factors $M_i = q_ik_i^\top$ and
      fixed identity values, the fiber over uniform
      averaging at the all-zero factors has the ideal 
      \begin{equation}\label{eq:serial-collapse-components}
(q_1k_1^\top) = (q_1) \cap (k_1),
\end{equation} with exactly
      two irreducible components $q_1 = 0$ or $k_1 = 0$. For
      $D \geq 2$, independent one-layer fibers have $2^D$
      components. Here collapse at the first layer leaves every
      later factor free. The same ideal
      also gives a depth-independent local learning
      coefficient $\lambda = \frac{d}{2}$ under the
      fixed-value loss model of Appendix~\ref{remark:collapse-learning}. 

      \section{Numerical weakness of exact identification}\label{sec:spectrum}

      The previous Taylor observations establish exact local
      recovery. We now study how sensitively outputs on
      finitely many inputs depend on the native parameters.
      Choose a finite input bank $\mathcal B$ in a
      neighborhood of zero where the network maps are
      defined. We collect their outputs with the finite
      observation map: 
      \begin{equation}\label{eq:finite-observation}
\mathcal O_{\mathcal B}(\theta) = \bigl(
        F_{\pi(\theta)}(X)\bigr)_{X \in \mathcal B}.
\end{equation}
      It has differential $D_{\theta}\mathcal O_{\mathcal
      B}$, which maps perturbations of the native weights to
      first-order changes in the observed outputs. Let
      $\theta(t)$ be an analytic parameter path, and write 
      \begin{equation}\label{eq:native-jacobian}
J_{\mathcal B}(t) = \left.D_{\theta}\mathcal
          O_{\mathcal B}\right|_{\theta=\theta(t)}.
\end{equation}

      The above differential is taken in all the native
      entries (including values) before restricting to the
      path $\theta(t)$. In particular, when differentiating
      along the path $\theta(t)$ itself, 
      \begin{equation}\label{eq:path-derivative}
\frac{\mathrm{d}}{\mathrm{d}t}\mathcal O_{\mathcal
        B}(\theta(t)) = J_{\mathcal B}(t)\dot\theta(t),
\end{equation}
        we test only the direction $\dot\theta(t)$. 

        The positive singular values of $J_{\mathcal B}(t)$
        quantify the strength of the first-order output
        changes. We measure how sensitivity deteriorates as
        $t \to 0$. We say a singular value $\sigma(t)$ has
        \emph{contact order} $\beta$ if 
        \begin{equation}\label{eq:contact-order}
c|t|^\beta\leq \sigma(t) \leq C|t|^\beta \qquad
          (0 < |t| \ll 1)
\end{equation} for $c, C > 0$ constant.
          Contact order zero means that the singular value
          stays bounded away from zero, and larger orders
          imply faster vanishing. 

      \begin{theorem}[Complete native contact spectrum]\label{thm-spectrum}
        Under the model assumptions of Section~\ref{sec:model}, let
        $Q_i, K_i \in \mathbb R^{d\times r_i}$, $1 \leq r_i
        \leq d$, and let all value matrices be trainable and
        $d\times d$: $p = d$. Suppose the normalizers have
        common first nonconstant Taylor degree $k$.
        Consider the native parameter path $\theta(t)$ given
        by 
        \begin{equation}\label{eq:collapse-path}
Q_i = tU_i, \qquad K_i = tW_i, \qquad V_i = I_d.
\end{equation} Assume $U_i$ and $W_i$ are of full column rank
          and that $\sym(U_iW_i^\top) \neq 0$ for all
          layers.

          Then there exists a fixed finite input bank
          $\mathcal B$ such that for all sufficiently small
          $t \neq 0$, the Jacobian $J_{\mathcal B}(t)$ in
          all native parameters has positive singular values
          comparable to $1$ with multiplicity $d^2$ and to $|t|^{\beta_i}$
          with multiplicity $r_i(2d - r_i)$ for $1 \leq i
          \leq D$, where 
          \begin{equation}\label{eq:native-orders}
\beta_i = 2k \cdot 3^{i-1} - 1.
\end{equation} 

          The exact source kernel dimension is 
          \begin{equation}\label{eq:native-kernel}
\zeta = (D-1)d^2 + \sum_{i=1}^D r_i^2.
\end{equation}
        \end{theorem}
\begin{figure*}[t]
\centering
\includegraphics[width=0.80\textwidth]{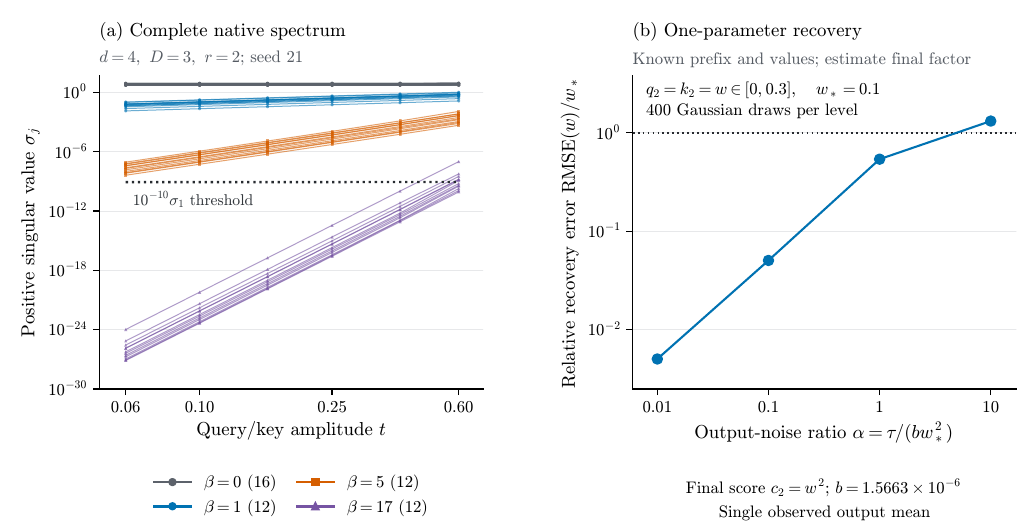}
\caption{\textbf{Exact rank and recovery under noise.}
(a) Positive native Jacobian singular values at 120-digit
precision, compared with the cutoff $10^{-10}\sigma_1$.
(b) Relative error in recovering the final nonnegative
factor from one noisy output mean with a known prefix,
using 400 trials per noise level. Here $b=mz^2$, where
$m$ and $z$ are the known prefix mean and contrast.
Full protocols are given in Appendix~\ref{appendix-experiments}.}
\label{fig:spectra}
\end{figure*}

        On complements to the exact gauge directions, divide
        each derivative group by its predicted power of $t$
        and take $t=0$. Any bank separating these
        limiting functions gives the conclusion, with constants
        depending on $\mathcal B$ and $\theta(t)$.
        Appendix~\ref{appendix-d} gives the equivalent Smith factors
        and heterogeneous-degree formula.

        For softmax with $d=4$, $D=3$, and $r_i=2$, the 96 native
        directions give 16 singular values of order zero, twelve each
        of orders 1, 5, 17, and 44 exact zeros. The zeros represent
        query/key factor gauges and intermediate value-coordinate
        freedom. Positive-order bands measure weakening sensitivity.

        The multiplicity $r_i(2d-r_i)$ is the dimension of the
        rank-$r_i$ matrix stratum. At full-column-rank factor pairs,
        the infinitesimal factor gauges are
        \begin{equation}\label{eq:infinitesimal-factor-gauge}
\begin{gathered}
 (\delta Q_i,\delta K_i)=(Q_iB_i,-K_iB_i^\top),\\
 B_i\in\mathbb R^{r_i\times r_i}.
\end{gathered}
\end{equation} and thus they have $r_i^2$ independent
          kernel directions per layer. 

          Along $\theta(t)$, Smith reduction over the ring of
          convergent series $\mathbb R\{t\}$
          refines functional dimension by recording the contact
          order of each positive singular-value band.

          These orders arise from cubic contrast propagation and
          differentiation of query/key factors. Appendix~\ref{appendix-d}
          proves the matching upper and lower bounds, while
          Appendix~\ref{appendix-c} constructs a finite separating bank.
          Efficient bank design remains open.

          Input scaling relates $T_N$ to $\mathcal O_{\mathcal B}$:
          with identity values and fixed scores $\mathbf B$,
          \begin{equation}\label{eq:scaling}
F_{t^2 \mathbf B}(X) = t^{-1}F_{\mathbf B}(tX),
            \quad t\neq 0.
\end{equation} 

          For the corresponding layer, $N_i = \beta_i + 2$.
          Note $t^2$ is a common inverse-temperature scale
          in softmax. 

The same spectrum quantifies local statistical information
on the fixed input bank. With independent Gaussian output
noise of variance $\tau^2$, the Fisher information is
\citep{rothenberg1971}
\begin{equation}\label{eq:fisher-information}
\mathcal I_{\mathcal B}(t)
=\tau^{-2}J_{\mathcal B}(t)^\top J_{\mathcal B}(t).
\end{equation}
Its positive eigenvalues have scales $\tau^{-2}$ and
$\tau^{-2}|t|^{2\beta_i}$, with the same multiplicities as
in Theorem~\ref{thm-spectrum}. Exact gauges give zero information, while identifiable
directions can carry arbitrarily little as the factors
approach collapse. Thus a small measured output change can
reflect either redundancy or weak local sensitivity. An
audit using such changes as evidence about internal
parameters must account for its input bank and noise level.

\section{Numerical evidence}\label{sec:experiments}

We test the predicted contact orders and their consequences
for numerical rank and recovery under noise. Full protocols and the
scalar plot are in Appendix~\ref{appendix-experiments}.

\paragraph{Depth-dependent powers.}
Depth-four scalar softmax at 110 digits gives final-pair
slopes $(0.999994,\allowbreak 5.002211,\allowbreak 17.001118,\allowbreak 53.000562)$ at
$t=0.04,0.02$, with maximum error below $0.0023$ from
$(1,5,17,53)$. The weakest singular value at $t=0.02$ is
$3.86\times10^{-92}$.

\paragraph{The complete matrix spectrum.}
For $n=d=4$, $D=3$, $r_i=2$, and 48 inputs (seed 21),
we get $J\in\mathbb R^{768\times96}$.
Under a cutoff of $10^{-10}\sigma_1$, float64 automatic
differentiation (AD) reports rank 40 at small scales
versus the exact 52. At 120 digits, the analytic Jacobian
resolves all twelve weak positive values
(Figure~\ref{fig:spectra}(a)).

\begin{center}
\small
\begin{tabular}{@{}>{\raggedright\arraybackslash}p{0.51\linewidth}r@{}}
\toprule
Check & Result \\
\midrule
Relative Frobenius error vs.\ AD & $\le1.47\times10^{-16}$ \\
Weak-band final-pair slopes & $[16.9943,17.0148]$ \\
Max.\ weak-space angle sine, $t=0.06$ & $5.21\times10^{-20}$ \\
90/120-digit relative change & $\le1.31\times10^{-35}$ \\
Smallest positive value, \mbox{$t=0.06$} & $7.58\times10^{-28}$ \\
\bottomrule
\end{tabular}
\end{center}

We compare the analytic Jacobian with AD where float64
arithmetic remains reliable. To identify the weakest band,
we project the final-layer derivative's image orthogonally
to the value and earlier-layer derivative images. The
small angle in the table shows that this image nearly
coincides with the weakest left singular subspace.
At $t=0.06$, the same cutoff still gives rank 40 at high
precision, since the twelve weak values lie below it.

\paragraph{One unknown parameter under noise.}
In Figure~\ref{fig:spectra}(b), all weights are known except
the final query/key factor $q_2=k_2=w\geq0$. We add
Gaussian noise to the output mean and estimate $w$ by
inverting the mean response, clipping observations outside
its attainable range to the endpoints. As the noise
standard deviation increases from $1.57\times10^{-10}$ to
$1.57\times10^{-7}$, the empirical relative error
$\operatorname{RMSE}(\hat w)/w_*$ rises from $0.0050$ to
$1.31$. Thus recovery from noisy observations can still be
difficult when the parameter is uniquely determined.

\section{Related work}\label{sec:related}

Classical neural-network identifiability recovers weights
from the realized function under assumptions about the
architecture and activation
\citep{sussmann1992,albertini1993,fefferman1994}.
\citet{vlacic2021} relate complete functional equivalence
to affine symmetries of the activation function. For ReLU
networks, parameter identifiability and functional dimension
distinguish parameter count from locally independent
functional directions \citep{bonapellissier2023,grigsby2025}.

Neuroalgebraic geometry also studies the parameter-to-function
map. \citet{kileel2019} analyze functional-variety dimension
for polynomial networks. \citet{shahverdi2025} establish
generic identification up to filter rescaling for monomial
convolutional networks. \citet{alexandr2026} derive coefficient
invariants for shallow unnormalized cubic lightning attention,
and \citet{pepinlehalleur2024} study matrix-multiplication
fibers through quiver representations. Our local normal form
extends these geometric descriptions to normalized attention
where finite observations give local effective coordinates.
At collapse, one layer can erase all later scores, so the
network fiber need not be a product of individual layer fibers.

For attention, \citet{brunner2020} study nonuniqueness of
attention weights at a fixed input, whereas
\citet{henry2025} formulate effective identification from
the entire function. \citet{meloux2025} study identifiability of explanations,
finding nonunique circuits and causal alignments in
small MLPs trained on Boolean tasks.
Query algorithms recover single-head
attention regressors \citep{bhattamishra2026} and canonical
multi-head representations \citep{kim2026} from real-valued
token queries.

\citet{dong2021} establish cubic contrast contraction and
doubly exponential rank collapse in pure attention.
\citet{noci2022} connect collapse to vanishing query/key
gradients, and \citet{geshkovski2023} study long-time token
clustering as particle interaction dynamics with fixed
weights. When a layer receives nearly identical tokens,
score changes can become difficult to distinguish. We
quantify this sensitivity loss at fixed depth as the
query/key factors approach zero. Contraction supplies upper
bounds, while matching lower bounds require a finite input
bank separating the divided derivative directions.

Future work includes the nonanalytic class
of the conjecture of Henry--Marchetti--Kohn, deeper softmax fibers with
singular mixed prefixes at three or more tokens, and general
exceptional fibers for zero-slope normalizers. It also
remains to study how these local product descriptions fit
together where the effective differential loses rank. Skip connections can counteract collapse
\citep{dong2021}, and causal and sparse masks can change
its rate \citep{wu2024}. Extending the native-spectrum
analysis to these architectures and multi-head attention
requires tracking their contrast dynamics. Since our experiments use fixed input sets and
parameter paths, the prevalence and consequences of weak
sensitivity in trained networks remain to be studied.
We also ask which internal observations or interventions
distinguish realizations within the same input--output fiber.
\label{main-text-end}
\newpage
\bibliography{main}
\newpage 
\appendix

\section{Identification}\label{appendix-a}

Throughout the appendix, we set $h_i=\log\phi_i$ and $g_i=h_i'$. We drop the index when it is clear we are discussing one layer.

We use the same factorization as introduced previously:
\[
 \Theta_{\mathrm{nat}}\xrightarrow{\pi}\Theta_{\mathrm{eff}}
 \xrightarrow{\mathcal R}\mathbb R\{X\}^{p\times n},\qquad
 \theta\longmapsto\eta=(L,\mathbf M)\longmapsto[F_\eta]_0.
\]
Here $\Theta_{\mathrm{nat}}$ is the space of native weights for the fixed architecture, and $\Theta_{\mathrm{eff}}=\mathbb R^{p\times d}\times(\mathbb R^{d\times d})^D$ is the ambient effective space. The effective tuples realizable by native weights form the image $\pi(\Theta_{\mathrm{nat}})$. We compute $\mathcal R^{-1}(\mathcal R(\eta))$ by comparing an arbitrary $\widetilde\eta$ with a marked tuple satisfying the hypotheses of Theorem~\ref{thm-identification}.

\subsection{Analytic setting and a one-layer expansion}\label{analytic-setting-and-a-one-layer-expansion}

Throughout, tokens are columns, output coordinates and token positions are labeled, and \[
S_M^\phi(X)_a=\frac{\sum_b\phi(x_a^\top Mx_b)x_b}{\sum_b\phi(x_a^\top Mx_b)}.
\] 

The normalizer above is positive, nonconstant, and real analytic, and normalized such that \(\phi(0)=1\). Since the normalization denominator equals \(n\) at $X=0$, it is a unit in the local analytic ring. The layer maps and their finite compositions are therefore jointly analytic in parameters and inputs near zero. In particular, we may differentiate their convergent expansions in the parameters and extract input coefficients in either order. We will use the following two operations in an analytic germ ring $\mathcal A$:
\[
 u(0)\ne0\ \Longrightarrow\ u^{-1}\in\mathcal A,
 \qquad ab=0,\ a\not\equiv0\ \Longrightarrow\ b=0.
\]
The first operation inverts a unit. The second cancels a nonzero germ because $\mathcal A$ is an integral domain. In particular, the factor $a$ may vanish at the origin. We use this cancellation to compare exact functions. Quantitative sensitivity requires the separate estimates in Appendix~\ref{appendix-d}. All function identities below are identities of input germs. For globally positive analytic normalizers, the identity theorem extends them to the connected full input domain. 

For any token matrix $Y$, we define its mean and contrast by
\[
 \mu(Y)=n^{-1}Y\mathbf1,\qquad
 \Delta(Y)=Y-\mu(Y)\mathbf1^\top.
\]
For two tokens, these are the mean $(y_++y_-)/2$ and the deviations $\pm(y_+-y_-)/2$. Since the attention weights sum to one, we have
\[
 S_M^\phi(m\mathbf1^\top)=m\mathbf1^\top,\qquad
 \{x_a\}\subset A\ \Longrightarrow\ \{S_M^\phi(X)_a\}\subset A
\]
for all affine subspaces $A$. On the line \(m+\mathbb Rz\), write \(x_a=m+s_az\), \(\sum_as_a=0\), and set \[
q=m^\top Mm,\quad u=z^\top Mm,\quad v=m^\top Mz,\quad w=z^\top Mz,
\quad \mu_j=n^{-1}\sum_a s_a^j.
\] Here $\mu_j$ is the $j$th moment of the centered scalar coordinates. The score in row \(a\) and column \(b\) is \(q+s_au+s_b(v+s_aw)\). Thus, setting \(q_a=q+s_au\) and \(v_a=v+s_aw\), we can write the output as $S_M^\phi(X)_a=m+T_az$, where \[
T_a=\frac{\sum_b s_b\phi(q_a+s_bv_a)}{\sum_b\phi(q_a+s_bv_a)}.
\] Expanding the numerator and denominator and using \(\mu_1=0\), we get \[
T_a=\mu_2g(q_a)v_a+
\frac{\mu_3}{2}\frac{\phi''(q_a)}{\phi(q_a)}v_a^2+O(\|s\|^4),
\qquad g=(\log\phi)'.
\] The degree-three part of the second summand is $\frac{\mu_3}{2}\frac{\phi^{\prime\prime}(q)}{\phi(q)}v^2$, which is the same for every query row. It therefore disappears when we subtract the output mean. To compute the remaining row dependence, we expand 
\[
g(q+s_au)(v+s_aw)
=g(q)v+s_a\{g(q)w+g'(q)uv\}+O(s_a^2),
\]
\[
 \begin{aligned}
 \mu(S_M^\phi(X))-m&=z\mu_2\mathsf H_M+O(\|s\|^3),\\
 \Delta(S_M^\phi(X))_a&=z\mu_2s_a\mathsf K_M+O(\|s\|^4),
 \end{aligned}
\]
where the signatures (these are the $H_M,K_M$ of the main text) are defined as 
\[
\mathsf H_M=g(m^\top Mm)m^\top Mz,
\qquad
\mathsf K_M=g(m^\top Mm)z^\top Mz
+g'(m^\top Mm)(z^\top Mm)(m^\top Mz).
\tag{A.1}\label{eq:app-a1}
\] Now, for an analytic incoming mean tending to \(m\), we assume
\[
 s_a=\varepsilon^ra c_a+O(\varepsilon^{r+1}),\qquad
 n^{-1}\sum_ac_a=0,\qquad n^{-1}\sum_ac_a^2=1.
\]
Then we have 
\[
 \begin{aligned}
 \mu(S_M^\phi(X))-\mu(X)&=\varepsilon^{2r}a^2z\mathsf H_M+O(\varepsilon^{2r+1}),\\
 \Delta(S_M^\phi(X))_a&=\varepsilon^{3r}a^3c_az\mathsf K_M+O(\varepsilon^{3r+1}).
 \end{aligned}
\]
These formulas give lower bounds on the orders: when a leading coefficient vanishes, the corresponding order increases. The remainders and their parameter derivatives are analytic and locally uniform. 
\subsection{A signature determines one score}\label{a-signature-determines-one-score}

Suppose $\mathsf H_M=\mathsf H_N$. 
If $g(0)\ne0$,
then the terms linear in $m$ give
\[
0=[\deg_m=1](\mathsf H_M-\mathsf H_N)
  =g(0)m^\top(M-N)z.
\]
Since this holds for all $m,z$, we obtain $M=N$.
Thus the signature determines every score in this case,
including those that are zero and skew-symmetric.

Now suppose $g(0)=0$ and that $\operatorname{sym}M\ne0$,
with $N$ still arbitrary. Let
\[
\phi(s)=1+bs^k+O(s^{k+1}),
\qquad b\ne0,\quad k\ge2,
\]
and set $q_M=m^\top Mm$, $q_N=m^\top Nm$.

On the diagonal $z=m$, the signature identity becomes
\[
 \chi(q_M)=\chi(q_N),\qquad
 \chi(s)=sg(s)=kb s^k+O(s^{k+1}).
\]
Consider its first homogeneous term, which gives the polynomial identity
\[
 q_M^k=q_N^k,\qquad
 0=q_N^k-q_M^k=\prod_{\zeta^k=1}(q_N-\zeta q_M)
 \quad\text{in }\mathbb C[m].
\]

Since $\mathbb C[m]$ is an integral domain, one factor
vanishes identically:
\[
q_N=\zeta q_M
\qquad\text{for some fixed }\zeta\text{ with }\zeta^k=1.
\]
Both quadratic forms have real coefficients, and $q_M\ne0$, so $\zeta$ is real. Thus
\[
q_N=\sigma q_M,\qquad
\begin{cases}
\sigma=1, & k\text{ odd},\\
\sigma\in\{1,-1\}, & k\text{ even}.
\end{cases}
\]
In particular,
\[
\operatorname{sym}N
=\sigma\,\operatorname{sym}M\ne0.
\]

For the positive branch,
\[
 q_N=q_M\ \Longrightarrow\
 g(q_M)m^\top(M-N)z=0\ \Longrightarrow\ m^\top(M-N)z=0
 \ \Longrightarrow\ M=N.
\]
Here $g(q_M)$ is a nonzero germ because its leading term is $kbq_M^{k-1}$. For the negative branch, $q_N=-q_M$ gives $\chi(q_M)=\chi(-q_M)$. Choose $v$ with $a:=q_M(v)\ne0$. Since $q_M(tv)=at^2$, we have 
\[
\chi(at^2)=\chi(-at^2)
\qquad\text{for all sufficiently small }t.
\]
The analytic function $F(s):=\chi(s)-\chi(-s)$ thus
has zeros accumulating at $0$, so we have
\[
F\equiv0,\qquad
0=\chi(s)-\chi(-s)=s\bigl(g(s)+g(-s)\bigr).
\]
Thus $g(s)=-g(-s)$ near $0$.

It follows that
\[
\frac{d}{ds}\{\log\phi(s)-\log\phi(-s)\}=g(s)+g(-s)=0.
\] The difference is zero at the origin, so $\phi$ is even. Substituting $q_N=-q_M$ and $g(-q_M)=-g(q_M)$ into the signature identity gives
\[
g(q_M)m^\top(M+N)z=0,
\]
and cancellation yields $N=-M$. 

Conversely, evenness preserves every attention weight:
\[
 \phi(s)=\phi(-s)\ \Longrightarrow\ S_M^\phi=S_{-M}^\phi.
\]
Thus the one-score signature fiber is $\{M\}$ when the normalizer is not even and $\{M,-M\}$ for an even normalizer.

\subsection{Layer isolation and proof of Theorem~\ref{thm-identification}}\label{layer-isolation-and-proof-of-theorem-1}

Take the line probes $x_a=m+\varepsilon c_az$ from Section~\ref{analytic-setting-and-a-one-layer-expansion}. We write $\mathsf K_i=\mathsf K_{M_i}$ for the contrast signature of layer $i$ and define the successive amplitudes by $a_0=1$ and \[
a_i=a_{i-1}^3\mathsf K_i.
\] We first show that these amplitudes remain nonzero under the symmetric-part hypothesis. Suppose \(\operatorname{sym}M_i\ne0\), and restrict the contrast signature to $z=m$. With \(s=m^\top M_im\) and \(\phi_i(s)=1+b_i s^{k_i}+\cdots\), we obtain \[
\mathsf K_i(m,m)=g_i(s)s+g_i'(s)s^2
=b_i k_i^2s^{k_i}+O(s^{k_i+1})\ne0.
\] Thus $\mathsf K_i$ is a nonzero germ. Since the analytic germ ring has no zero divisors, the amplitude recursion now gives $a_i\not\equiv0$ at every such layer. Writing $X^{(i)}=S_{M_i}^{\phi_i}(X^{(i-1)})$, $X^{(0)}=X$, and $\mu_i=\mu(X^{(i)})$, the one-layer expansion gives 
\[
 \Delta(X^{(i)})_a=\varepsilon^{3^i}a_ic_az+O(\varepsilon^{3^i+1}),\qquad
 a_i\not\equiv0,\qquad \mu_0=m.
\]

Once layers $1,\ldots,j-1$ have been recovered, their
output mean $\mu_{j-1}$ is known. Subtracting its readout
from the observed output mean gives
\[
\operatorname{mean}F-L\mu_{j-1}
=L(\mu_D-\mu_{j-1})
=L\sum_{i=j}^D(\mu_i-\mu_{i-1}).
\]
The incoming contrast at layer $i$ has order $3^{i-1}$, so its mean increment satisfies
\[
\operatorname{ord}_\varepsilon(\mu_i-\mu_{i-1})\ge2\cdot3^{i-1}.
\]

For $i>j$ this bound is strictly larger than $2\cdot3^{j-1}$. Thus only layer $j$ can contribute to that coefficient, giving 
\[
[\varepsilon^{2\cdot3^{j-1}}]
\bigl(\operatorname{mean}F-L\mu_{j-1}\bigr)
=a_{j-1}^2Lz\mathsf H_j.
\tag{A.2}\label{eq:app-a2}
\] For any competing score with the same recovered prefix, the one-layer expansion gives $\operatorname{ord}_\varepsilon\widetilde\Delta^{(j)}\ge3^j$. Its later mean increments therefore start at order at least $2\cdot3^j$. A vanishing leading contrast coefficient can only increase this order, so \eqref{eq:app-a2} also isolates the competing next signature. 

Let $\mathcal R(\widetilde\eta)=\mathcal R(\eta)$. At consensus, we have \(F(m,\ldots,m)=(Lm,\ldots,Lm)\), and hence $\widetilde L=L$. Suppose that the first $j-1$ layer maps agree, allowing the exact even-normalizer signs: then their prefix means and amplitudes coincide. Choose a nonzero row $\ell$ of $L$. Subtracting their instances of \eqref{eq:app-a2} gives in the ring $\mathbb R\{m,z\}$
\[
 \underbrace{a_{j-1}^2(\ell z)}_{\not\equiv0}
 (\mathsf H_{M_j}-\mathsf H_{\widetilde M_j})=0
 \quad\Longrightarrow\quad
 \mathsf H_{M_j}=\mathsf H_{\widetilde M_j}.
\]
Section~\ref{a-signature-determines-one-score} now identifies $\widetilde M_j$ and shows that its symmetric part is nonzero. Any remaining sign preserves the layer map, so the two realizations have the same prefix through layer $j$. This completes the induction. Conversely, independently changing the signs at even normalizers preserves every layer map and hence the full function. We therefore obtain the fiber in Theorem~\ref{thm-identification}: 
\[
 \mathcal R^{-1}(\mathcal R(\eta))
 =\{L\}\times\prod_{i=1}^D\Gamma_iM_i,
 \qquad
 \Gamma_i=\begin{cases}\{1\},&\phi_i\text{ not even},\\
 \{1,-1\},&\phi_i\text{ even}.
 \end{cases}
\]

If $g_D(0)\ne0$, the term $[\deg_m=1]\mathsf H_{M_D}=g_D(0)m^\top M_Dz$ identifies any final score. Since there is no further layer to recover, we need no nonzero contrast after layer $D$, and $M_D$ may be arbitrary. 
Note that the separate zero-readout fiber is
\[
 L=0\quad\Longrightarrow\quad
 \mathcal R^{-1}(0)=\{0\}\times(\mathbb R^{d\times d})^D.
\]

Since the output on a consensus input $(m, \ldots, m)$ is $(\widetilde Lm,\ldots,\widetilde Lm)$, we note that realizing the
zero function requires $\widetilde Lm=0$ for every $m$
near zero, and thus $\widetilde L=0$.

\subsection{Native reduction and genericity}\label{native-reduction-and-genericity}

\begin{corollary}[Generic native identification]\label{cor-generic}

Fix all positive token widths and positive head widths in a pure serial native architecture, and let query, key, and value entries vary freely. Suppose we have known normalizers, as in  Theorem~\ref{thm-identification}. The native parameters with effective tuples having $L = 0$ or $\sym{M_i} = 0$ for some $i$ form a proper real algebraic subset. 
Outside this subset, Theorem~\ref{thm-identification} gives the complete effective fiber. Thus in every such architecture, positive injective real-analytic normalizers have generically singleton effective fibers. 
\end{corollary}

For a native layer \(X\mapsto V_iS^{\phi_i}_{A_i}(X)\), where \(A_i=Q_iK_i^\top\), we note that bilinearity gives \[
S^{\phi}_{A}(PX)=P S^{\phi}_{P^\top AP}(X),
\] including for $P$ rectangular or singular. When \(P_0=I\) and \(P_i=V_i\cdots V_1\), the effective parameters are \[
L=P_D,\qquad M_i=P_{i-1}^\top A_iP_{i-1}.
\tag{A.3}\label{eq:app-a3}
\] Thus the map $\pi$ is polynomial, and the exceptional native set is the algebraic preimage
\[
 \mathcal Z_{\mathrm{nat}}=\pi^{-1}(\mathcal Z_{\mathrm{eff}}),\qquad
 \mathcal Z_{\mathrm{eff}}=\{L=0\}\cup\bigcup_i\{\operatorname{sym}M_i=0\}.
\]

Let $d_0,\ldots,d_D$ be the layer dimensions, and let
$r_i\ge1$ be the prescribed head widths. We write
$e^{(i)}$ for the first standard basis vector of
$\mathbb R^{d_i}$. Choose
\[
V_i=e^{(i)}(e^{(i-1)})^\top,
\]
such that each value map preserves the first coordinate, 
and sends every other coordinate to zero. Set
\[
Q_i=K_i=
\begin{bmatrix}e^{(i-1)}&0&\cdots&0\end{bmatrix}
\in\mathbb R^{d_{i-1}\times r_i}.
\]
Then we have
\[
A_i=Q_iK_i^\top
=e^{(i-1)}(e^{(i-1)})^\top.
\]
Composition of the value maps gives
\[
P_i=V_i\cdots V_1=e^{(i)}(e^{(0)})^\top
\qquad(i\ge1),
\]
and thus
\[
\begin{aligned}
L&=e^{(D)}(e^{(0)})^\top\ne0,\\
M_i&=P_{i-1}^\top A_iP_{i-1}
=e^{(0)}(e^{(0)})^\top
\qquad(1\le i\le D).
\end{aligned}
\]
Every $M_i$ is symmetric and nonzero. Thus the constructed native tuple $\theta_*$ satisfies
\[
\pi(\theta_*)\notin\mathcal Z_{\mathrm{eff}}.
\]
Since $\pi$ is polynomial, we note that
$\mathcal Z_{\mathrm{nat}}=\pi^{-1}(\mathcal Z_{\mathrm{eff}})$
is therefore a proper algebraic subset, and its complement
is Zariski open dense and has full Lebesgue measure. 

On the complement, Theorem~\ref{thm-identification}
gives the complete effective fiber.

The reduction to effective parameters uses only the products
\[
L=P_D,\qquad M_i=P_{i-1}^\top Q_iK_i^\top P_{i-1},
\]
which remain well defined even for rectangular or rank-deficient
value maps. The native coordinate change in
Proposition~\ref{prop-native-geometry} additionally requires
square invertible value maps, since recovering the original
factors from the transformed factors requires the inverses in
\[
Q_i=P_{i-1}^{-\top}\widetilde Q_i,\qquad
K_i=P_{i-1}^{-\top}\widetilde K_i.
\]
The rank-one construction above establishes genericity
without requiring this additional invertibility.

\section{Exceptional fibers}\label{appendix-b}

We use the consensus subspace $\mathcal C$ and collapse loci $ Z_i$ defined in Section~\ref{sec:collapse}. A score belongs to the collapse locus precisely when all output tokens agree for every input near zero:
\[
 M\in Z_i
 \quad\Longleftrightarrow\quad \Delta\circ S_M^{\phi_i}\equiv0
 \quad\Longleftrightarrow\quad \operatorname{im}S_M^{\phi_i}\subset\mathcal C.
\]
Since every later layer fixes $\mathcal C$, note that changing the suffix after a collapsing layer preserves the realized function. This proves one inclusion in~\eqref{eq:first-collapse-fiber}.

For the reverse inclusion, we first show how a nonzero incoming contrast identifies the next score. We will then verify that every noncollapsing prefix supplies such a contrast in each regime of Theorem~\ref{thm-exceptions}. Suppose $\mathcal R(\widetilde\eta)=\mathcal R(\eta)$ and $L\ne0$. Consensus gives $\widetilde L=L$. Once the first $i-1$ layer maps have been recovered, write their outgoing centered tokens on a line probe as
\[
 \Delta_a^{(i-1)}=\varepsilon^r a c_a z+O(\varepsilon^{r+1}),
 \qquad a\not\equiv0.
\]
For the next layer, we write
\[
 \mathsf H_{i,M}(m,z)=h_i'(m^\top Mm)\,m^\top Mz.
\]
Then subtraction of the exact prefixes gives
\[
 [\varepsilon^{2r}]
 \bigl(\overline F_\eta-L\mu_{i-1}\bigr)
 =a^2Lz\,\mathsf H_{i,M_i}.
\]
The outgoing contrast has order at least $3r$, so every further mean increment has an order greater than $2r$. This remains true for degenerate competing next layers as well. Choose a nonzero row $\ell$ of $L$. Then by equality of the realized functions, we have 
\[
 a^2(\ell z)
 \bigl(\mathsf H_{i,M_i}-\mathsf H_{i,\widetilde M_i}\bigr)=0.
\]
The factors $a$ and $\ell z$ are nonzero germs, so cancellation gives equality of the two signatures. Their terms of degree one in $m$ are
\[
 [\deg_m=1]\mathsf H_{i,M}=h_i'(0)m^\top Mz.
\]
Since we have $h_i'(0)\ne0$, this identifies $M_i=\widetilde M_i$, including the case where $M_i$ collapses. For $L=0$, inputs at consensus force $\widetilde L=0$, and thus every score tuple realizes the zero function.

It remains to determine $ Z_i$ and to show that each preceding noncollapsing layer transmits a nonzero contrast germ. We prove this in the following two subsections. 

\subsection{Nonzero logarithmic slope and curvature}\label{nonzero-logarithmic-slope-and-curvature}

Assume we have $h_i'(0)h_i''(0)\ne0$ at every layer. For a single layer, by Appendix~\ref{appendix-a} we have that
\[
 \mathsf K_M(m,z)=h'(m^\top Mm)z^\top Mz
 +h''(m^\top Mm)(z^\top Mm)(m^\top Mz).
\]
Suppose $\operatorname{sym}M\ne0$. Then we have 
\[
 \mathsf K_M(0,z)=h'(0)z^\top Mz\not\equiv0.
\]
In the case that $M\ne0$ is skew, we have $m^\top Mm=z^\top Mz=0$ and $z^\top Mm=-m^\top Mz$, and thus
\[
 \mathsf K_M(m,z)=-h''(0)(m^\top Mz)^2\not\equiv0.
\]
Thus every nonzero score transmits a nonzero contrast germ, but
\[
 S_0^\phi(X)=\mu(X)\mathbf1^\top.
\]
Then $ Z_i=\{0\}$. This identifies every score through the first zero, establishing~\eqref{eq:first-collapse-fiber} in Theorem~\ref{thm-exceptions}(\ref{thm-exceptions:curvature}).

\subsection{Two tokens and the exact collapsing condition}\label{two-tokens-and-the-exact-collapsing-condition}

We write the two inputs as \(m\pm z\). For a skew score, let \(b=m^\top Mz\) and \(h=\log\phi\). Then the four scores are \(0,-2b,2b,0\), and from direct normalization we get \[
y_+=m-z\tanh(h(-2b)/2),\qquad
y_-=m+z\tanh(h(2b)/2).
\] Thus \[
\frac{y_+-y_-}{2}=zD_h(2b),\qquad
D_h(u)=-\frac12\{\tanh(h(u)/2)+\tanh(h(-u)/2)\}.
\tag{B.1}\label{eq:app-b1}
\] Since $\tanh$ is odd, the two terms cancel exactly when $\tanh(h(u)/2)=\tanh(-h(-u)/2)$. Local injectivity then gives
\[
 D_h\equiv0\ \Longleftrightarrow\ h(u)+h(-u)=0
 \ \Longleftrightarrow\ h\text{ is odd}.
\]
For a nonzero skew $M$, choose $u,v$ with $u^\top Mv\ne0$. Then $m=tu$ and $z=tv$ give $m^\top Mz=t^2u^\top Mv$. These values accumulate at zero, so analyticity makes vanishing of the contrast equivalent to $D_h\equiv0$. For a nonskew score with $h'(0)\ne0$, the constant term in $m$, $h'(0)z^\top Mz$, of $\mathsf K_M$ is nonzero. Then
\[
 \Delta\circ S_M^\phi\equiv0
 \ \Longleftrightarrow\
 M=0\ \text{ or }\ (M\ne0,\ M^\top=-M,\ h\text{ odd}),
\]
Combining these cases, we obtain the collapse locus in Theorem~\ref{thm-exceptions}(\ref{thm-exceptions:two-tokens}): $Z_i$ is the space of skew matrices when $h_i$ is odd, and $\{0\}$ otherwise.

Suppose $h$ is not odd, and let $c_\ell s^\ell$ be its first nonzero even term. Write $h=h_{\mathrm{odd}}+h_{\mathrm{even}}$. In \eqref{eq:app-b1}, the contributions from $h_{\mathrm{odd}}$ cancel. Taylor expansion in $h_{\mathrm{even}}/2$, using $\tanh^{\prime}(h_{\mathrm{odd}}/2)=1+O(s^2)$, then gives \[
D_h(u)=-\frac{c_\ell}{2}u^\ell+O(u^{\ell+2}).
\tag{B.2}\label{eq:app-b2}
\] For incoming half-spread \(\varepsilon^ra z+O(\varepsilon^{r+1})\) and mean tending to \(m\), the centered output has order \((\ell+1)r\), with nonzero coefficient \[
-2^{\ell-1}c_\ell a^{\ell+1}(m^\top Mz)^\ell z.
\] At a noncollapsing layer, the outgoing centered order is thus
\[
 r'=\begin{cases}3r,&\operatorname{sym}M\ne0,\\
 (\ell+1)r,&M\ne0\text{ skew},\end{cases}
 \qquad r'>r,
\]
with a nonzero analytic coefficient. The leading mean increment is still \(\varepsilon^{2r}a^2z\mathsf H_M\), and thus the argument continues through all noncollapsing prefix layers and identifies the first collapsing score. 
After the first collapsing layer, every later input lies in $\mathcal C$, so every later score is free. The same order comparison applies to an arbitrary competing next layer: faster contraction only delays its later mean increments. This proves~\eqref{eq:first-collapse-fiber} in Theorem~\ref{thm-exceptions}(\ref{thm-exceptions:two-tokens}).

\subsection{Softmax collapse for every token count}\label{softmax-collapse-for-every-token-count}

For softmax, let \(x_a=m+\varepsilon y_a\), \(\sum_a y_a=0\), and \(C=n^{-1}\sum_a y_ay_a^\top\). The row-constant scores all cancel, and the remaining logits are \[
\varepsilon m^\top My_b+\varepsilon^2 y_a^\top My_b.
\] We expand the normalized exponentials through order two, noting that the first term depending on the query index is \(n^{-1}\varepsilon^2 y_a^\top My_b\). The remaining terms of that order are query independent, and thus multiplication by the values and averaging gives  \[
\operatorname{mean}S_M(X)=m+\varepsilon^2CM^\top m+O(\varepsilon^3),
\] \[
S_M(X)_a-\operatorname{mean}S_M(X)
=\varepsilon^3CM^\top y_a+O(\varepsilon^4).
\tag{B.3}\label{eq:app-b3}
\] In the case that \(\operatorname{sym}M\ne0\), choose a collinear cloud in a direction \(z\) with \(z^\top Mz\ne0\), which means that the centered coefficient is nonzero. If \(M\ne0\) is skew and \(n\ge3\), we choose \(y,y'\) with \(\kappa=y^\top My'\ne0\), and use centered vectors \(y_1=y,y_2=y',y_3=-y-y'\), with any other remaining vectors being zero. Then \[
CM^\top y_1=\frac\kappa n(y+2y')\ne0.
\] With the two-token calculation, this gives the complete single-layer image criterion 
\[
 \operatorname{im}S_M^{\exp}\subset\mathcal C
 \quad\Longleftrightarrow\quad
 \begin{cases}M=0,&n\ge3,\\M^\top=-M,&n=2.\end{cases}
\]

To see why the skew case requires more than a line probe, set $y_a=c_az$ and $\gamma=m^\top Mz$. The logits become $\varepsilon\gamma(c_b-c_a)$, and the term $-\varepsilon\gamma c_a$ cancels within each row:
\[
 A_{ab}=\frac{e^{\varepsilon\gamma c_b}}{\sum_ce^{\varepsilon\gamma c_c}},
\]
Thus every row has the same weights and every collinear cloud collapses. The three-token construction above retains contrast by using two independent directions. Section~\ref{recovery-from-generic-two-plane-restrictions} uses these directions throughout a stack to recover scores after nonzero skew prefixes.

\subsection{Why nonzero slope cannot simply be removed}\label{why-nonzero-slope-cannot-simply-be-removed}

Take two tokens, \(\phi_1(s)=\exp(s^2)\), a nonzero skew \(M_1\), and a second score \(M_2=0\). Here \(h_1(s)=s^2\) is even. Formula \eqref{eq:app-b1} gives \[
y_+=m-z\tanh(2b^2),\qquad
y_-=m+z\tanh(2b^2).
\] Hence
\[
 \mu(S_{M_1}^{\phi_1}(X))=m,\qquad
 \Delta(S_{M_1}^{\phi_1}(X))\not\equiv0,\qquad
 S_0^{\phi_2}S_{M_1}^{\phi_1}(X)=m\mathbf1^\top.
\]
The first layer is noncollapsing, yet its score is invisible after the second layer. The two normalizer conditions serve different purposes:
\[
 \phi\text{ even}\ \Longrightarrow\ S_M^\phi=S_{-M}^\phi,
 \qquad \log\phi\text{ odd}\ \Longrightarrow\ \text{two-token skew collapse}.
\]
Here $h_1^{\prime}(0)=0$: the first score affects the contrast but not the mean, and the next layer removes that contrast. Thus the first-collapse classification depends on the nonzero-slope assumption, which lets the mean signature identify the collapsing score itself.

\section{Finite jets}\label{appendix-c}

The map $T_N=j_0^N\circ\mathcal R$ records the Taylor coefficients of the realized function through total input degree $N$ at $X=0$. We ask both when these coefficients give local coordinates and when they determine the full realization fiber:
\[
 \begin{array}{ll}
 \text{local coordinates:}&\ker DT_N(\eta)=0,\\[2pt]
 \text{complete-fiber determination:}&T_N^{-1}(T_N(\eta))=\mathcal R^{-1}(\mathcal R(\eta)).
 \end{array}
\]
The first gives a local inverse onto the image. To identify the complete fiber, we must also distinguish any distant sign branches. We establish the local order, then determine when the entire finite-observation fiber is already the realization fiber. Here we scale the whole input as $X=\rho Y$, whereas the parameter $\varepsilon$ in Appendix~\ref{appendix-a} scaled only the contrast about a fixed mean. We use $\alpha_i$ for total input order and $a_i$ for the analytic spread amplitude.

\subsection{An order estimate for arbitrary clouds}\label{an-order-estimate-for-arbitrary-clouds}

\begin{proposition}[Heterogeneous local jet orders]\label{prop-heterogeneous-jets}

Let \(\log\phi_i(s)=b_i s^{k_i}+O(s^{k_i+1})\), \(b_i\ne0\), and define \[
\alpha_0=1,\qquad \alpha_i=3\alpha_{i-1}+2k_i-2,
\qquad N_i=2\alpha_{i-1}+2k_i-1.
\] If \(L\ne0\) and \(\operatorname{sym}M_i\ne0\) for every layer, the degree-\(N_D\) zero-input jet is an immersion in \((L,\mathbf M)\) and has a local analytic inverse onto its image. Every lower jet is independent of \(M_D\). If \(k_D=1\), the final matrix may be arbitrary. For common \(k_i=k\), \(N_D=2k3^{D-1}+1\).

\end{proposition}

Let a layer receive mean \(\bar x=O(\rho)\) and centered tokens \(\Delta_a=O(\rho^\alpha)\), \(\alpha\ge1\). Write \[
x_a^\top Mx_b=\bar s_a+\widetilde s_{ab},\qquad
\bar s_a=x_a^\top M\bar x=O(\rho^2),\quad
\widetilde s_{ab}=x_a^\top M\Delta_b=O(\rho^{\alpha+1}).
\] Since $\phi=e^h$, the factor $e^{h(\bar s_a)}$ is common to every weight in row $a$ and cancels in the normalization. Thus we may replace the log weights by $h(\bar s_a+\widetilde s_{ab})-h(\bar s_a)$. If $h(s)=bs^k+O(s^{k+1})$, every monomial in this difference satisfies, for $r\ge k$ and $1\le j\le r$,
\[
 (\bar s_a+\widetilde s_{ab})^r-\bar s_a^r
 =\sum_{j=1}^r\binom rj\bar s_a^{r-j}\widetilde s_{ab}^{\,j},\qquad
 2(r-j)+j(\alpha+1)\ge\alpha+2r-1.
\]
Writing $A_{ab}$ for the normalized attention weights, we therefore have $A_{ab}-1/n=O(\rho^{\alpha+2k-1})$. Since the centered tokens satisfy $\sum_b\Delta_b=0$,
\[
 x_a'-\bar x=\sum_b(A_{ab}-1/n)\Delta_b=O(\rho^{2\alpha+2k-1}).
\tag{C.1}\label{eq:app-c1}
\] To bound the contrast, compare output rows $a$ and $c$. Their score difference contains $x_a-x_c=\Delta_a-\Delta_c=O(\rho^\alpha)$ in place of an $O(\rho)$ query factor. Thus every row-dependent monomial gains $\alpha-1$ powers of $\rho$. The normalization denominators are analytic units, so the same gain holds for the normalized weights. Averaging these row differences gives centered output
\[
O(\rho^{3\alpha+2k-2}).
\tag{C.2}\label{eq:app-c2}
\] In both formulas, the output is the stated power of $\rho$ times an analytic function of $\rho$, the parameters, and the scaled input. Thus the bounds remain valid where a leading coefficient vanishes. The mean remains $O(\rho)$.

Starting from \(X=\rho Y\), the order propagation is
\[
 \alpha_i=3\alpha_{i-1}+2k_i-2,\qquad
 N_j=2\alpha_{j-1}+2k_j-1.
\]
Equation~\eqref{eq:app-c2} bounds every outgoing contrast by order $\alpha_i$, while \eqref{eq:app-c1} bounds the first change caused by $M_j$, with its prefix fixed, by order $N_j$. Bounded analytic derivatives of the suffix cannot lower that difference order. Therefore
\[
 N<N_D\quad\Longrightarrow\quad D_{M_D}T_N\equiv0.
\]

\subsection{Immersion and the sharp local order}\label{immersion-and-the-sharp-local-order}

The order bounds above apply to all inputs. To show that they are attained and identify the parameters, we return to the line probes of Appendix~\ref{appendix-a}. Let $v_i$ be the least degree in $m$ of the nonzero prefix amplitude $a_i$. If \(\operatorname{sym}M_i\ne0\), the least-mean-degree part of \(\mathsf K_i\) is \[
\kappa_i=b_i k_i\bigl[q_i^{k_i-1}z^\top M_iz
+(k_i-1)q_i^{k_i-2}(z^\top M_im)(m^\top M_iz)\bigr],
\quad q_i=m^\top M_im,
\] where the second summand is omitted for \(k_i=1\). At \(z=m\) it equals \(b_i k_i^2q_i^{k_i}\), so it is nonzero. Thus \[
v_0=0,\quad v_i=3v_{i-1}+2k_i-2,\quad
A_i=A_{i-1}^3\kappa_i\ne0,
\] where \(A_i\) is the lowest-mean-degree part of \(a_i\). Since \(a_i\) is homogeneous of degree \(3^i-1\) in \(z\), \(\alpha_i=3^i+v_i\). The part of \eqref{eq:app-a2} of mean degree \(2v_{j-1}+2k_j-1\) is \[
b_jk_j A_{j-1}^2 Lz\,P_{k_j}(M_j;m,z),
\quad
P_k(M;m,z)=(m^\top Mm)^{k-1}m^\top Mz.
\tag{C.3}\label{eq:app-c3}
\] Its spread degree is \(2\cdot3^{j-1}\), so its total input degree is exactly \(N_j\).

We first show that this coefficient detects every infinitesimal change in the current score. For $k>1$ and a matrix variation $T$, differentiating $P_k$ gives \[
DP_k(M)[T]=(k-1)(m^\top Mm)^{k-2}(m^\top Tm)(m^\top Mz)
+(m^\top Mm)^{k-1}m^\top Tz.
\] We compute its kernel by two polynomial cancellations:
\[
 DP_k(M)[T]=0
 \ \Longrightarrow\ k(m^\top Mm)^{k-1}m^\top Tm=0
 \ \Longrightarrow\ \operatorname{sym}T=0
 \ \Longrightarrow\ m^\top Tz=0
 \ \Longrightarrow\ T=0.
\]
Indeed, setting $z=m$ gives the first equality. Since $m^\top Mm$ is a nonzero polynomial, cancellation forces $m^\top Tm=0$, or $\operatorname{sym}T=0$. Substituting this into the full derivative leaves $(m^\top Mm)^{k-1}m^\top Tz=0$, and a second cancellation gives $T=0$. For \(k=1\), \(P_1=m^\top Mz\) is linear and injective at every base matrix.

Now let $\dot\eta=(\delta L,\delta M_1,\ldots,\delta M_D)$ lie in $\ker DT_{N_D}(\eta)$. The degree-one consensus output gives $\delta L=0$. Suppose inductively that $\delta M_i=0$ for $i<j$. Then the prefix coefficient $A_{j-1}$ has zero variation, and differentiating \eqref{eq:app-c3} gives
\[
 b_jk_jA_{j-1}^2(Lz)DP_{k_j}(M_j)[\delta M_j]=0
 \ \Longrightarrow\ DP_{k_j}(M_j)[\delta M_j]=0
 \ \Longrightarrow\ \delta M_j=0.
\]
To cancel $Lz$, we take a nonzero row $\ell$ of $L$. Then $A_{j-1}^2\ell z$ is a nonzero polynomial. Later layers have higher spread degree and do not enter this equation. Each step uses an available coefficient of $T_{N_D}$ because
\[
N_{j+1}-N_j=4\alpha_{j-1}+2k_j+2k_{j+1}-4>0.
\] The induction gives $\ker DT_{N_D}(\eta)=0$. We can therefore select as many observed Taylor coefficients as there are effective parameters so that their Jacobian is nonsingular at the marked tuple. Applying the analytic inverse-function theorem to these selected coefficients gives a local inverse onto the observation image. Equation~\eqref{eq:app-c1} shows that every lower jet is independent of $M_D$, proving sharpness. For $k_D=1$, the final derivative is injective at any matrix. For common $k$,
\[
 \alpha_i=k3^i+1-k,\qquad N_D=2k3^{D-1}+1.
\]
This proves Proposition~\ref{prop-heterogeneous-jets} and the local-order part of Theorem~\ref{thm-jets}.

\subsection{Nonzero slope: rational local inversion and uniform determination}\label{nonzero-slope-rational-local-inversion-and-uniform-determination}

If every \(g_i(0)\ne0\) and the prefix matrices have nonzero symmetric parts, the coefficient in \eqref{eq:app-a2} of mean degree one is \[
a_{j-1}(0,z)^2g_j(0)Lz\,(m^\top M_jz),
\qquad
a_{j-1}(0,z)=\prod_{i<j}\{g_i(0)z^\top M_iz\}^{3^{j-1-i}}.
\tag{C.4}\label{eq:app-c4}
\] Every prefix factor is a nonzero polynomial. Comparing two tuples and canceling that factor recovers each arbitrary next matrix at total degree $2\cdot3^{j-1}+1$. Thus, with $N_D=2\cdot3^{D-1}+1$,
\[
 T_{N_D}^{-1}(T_{N_D}(\eta))
 =\mathcal R^{-1}(\mathcal R(\eta))=\{\eta\},
\]
including an arbitrary final matrix.

We can make this local recovery explicit. Choose a nonzero row $l$ of $L$ and independent $z_1,\ldots,z_d$ satisfying
\[
 lz_b\ne0,\qquad z_b^\top M_iz_b\ne0\quad(i<D,\ 1\le b\le d).
\]
The product of the displayed factors is a nonzero polynomial. We can therefore choose each $z_b$ outside its zero set and outside the span of the vectors already chosen. Once chosen at the marked tuple, these directions remain fixed throughout the local inversion. Let $c_{j,b}$ be the column vector such that the selected output row of \eqref{eq:app-c4}, evaluated at $z_b$, is $m^\top c_{j,b}$. Then
\[
 M_jz_b=\frac{c_{j,b}}{a_{j-1}(0,z_b)^2g_j(0)(lz_b)},
\]
and
\[
M_j=(M_jz_1,\ldots,M_jz_d)(z_1,\ldots,z_d)^{-1}.
\] We first recover $L$ linearly from consensus and then apply these formulas successively to each layer. The Taylor and exact-prefix coefficients are polynomial in the parameters, with fixed normalizer coefficients as constants. Every displayed denominator is nonzero at the marked tuple and remains so nearby. The recursion therefore defines a rational map $G$ on an observation chart with
\[
 G\circ T_{N_D}=\operatorname{id}.
\]
This is the remaining claim of Theorem~\ref{thm-jets}.

\begin{corollary}[Uniform finite determination with nonzero curvature]\label{cor-full-jet}

If \(g_i(0)g_i'(0)\ne0\) for every layer, the zero-input jet through \[
N_D^{\mathrm{all}}=4\cdot3^{D-1}-1
\] determines the complete effective fiber at every tuple, including all zero-score and zero-readout exceptions. It is an injective immersion on the locus \(L\ne0\), \(M_i\ne0\) for \(i<D\). The displayed degree is sufficient for arbitrary \(n\), while optimality for noncollinear clouds is not asserted.

\end{corollary}

Now suppose \(g_i(0)g_i'(0)\ne0\), permitting nonzero skew prefixes. The least mean degree of \(\mathsf K_i\) is zero at a nonskew score and two at a nonzero skew score, the latter because \(\mathsf K_i=-g_i'(0)(m^\top M_iz)^2\). Consequently \[
v_0=0,\qquad v_i=3v_{i-1}+2\mathbf1\{M_i\text{ is skew}\},
\qquad v_i\le3^i-1.
\] The lowest-mean-degree part of \eqref{eq:app-a2} identifies the next arbitrary matrix using \[
N_j^{\mathrm{ad}}=2\cdot3^{j-1}+2v_{j-1}+1
\le4\cdot3^{j-1}-1.
\tag{C.5}\label{eq:app-c5}
\] The coefficient is injective in the next matrix even at zero. Induction thus recovers the prefix through first zero, or every score if none occurs, and
\[
 T_{N_D^{\mathrm{all}}}^{-1}(T_{N_D^{\mathrm{all}}}(\eta))
 =\mathcal R^{-1}(\mathcal R(\eta)).
\]
For $L=0$, the degree-one consensus output forces every competing readout to vanish, giving the entire zero-function fiber. If $L\ne0$ and every prefix score is nonzero, differentiating the recovery equations identifies all variations in order, so $\ker DT_{N_D^{\mathrm{all}}}=0$. A nonsingular minor stays nonsingular nearby even if a skew score becomes nonskew. This proves Corollary~\ref{cor-full-jet}. In the next subsection, we determine the optimal order for two tokens.

\subsection{Exact adaptive order for two tokens}\label{exact-adaptive-order-for-two-tokens}

\begin{proposition}[Exact adaptive jet order for two tokens]\label{prop-adaptive}

Assume \(n=2\), \(L\ne0\), and \(h_i'(0)\ne0\) for all \(i\). Let \(s\) be the first collapsing layer if one occurs before the end, and put \(s=D\) otherwise. Along the noncollapsing prefix set \(\alpha_0=1\) and \[
\alpha_i=\begin{cases}
3\alpha_{i-1},&\operatorname{sym}M_i\ne0,\\
(\ell_i+1)\alpha_{i-1}+\ell_i,
&M_i\ne0\text{ skew},
\end{cases}
\] where \(\ell_i\) is the first nonzero even degree of \(h_i\). This degree is finite in the second case because the skew prefix is assumed to be noncollapsing. Then \(2\alpha_{s-1}+1\) is the least zero-input jet order that determines the complete fiber at the marked tuple. If no prefix layer collapses, that jet is also an immersion. In the nonzero-curvature case, \(\alpha_i=3^i+v_i\), where \[
v_0=0,\qquad v_i=3v_{i-1}+2\mathbf1\{M_i\text{ is skew}\}.
\]

\end{proposition}

Assume the hypotheses of \hyperref[prop-adaptive]{the adaptive proposition}. Every input pair can be written as \(\rho m\pm\rho z\). After a noncollapsing prefix the pair has mean \(\rho m+\text{higher terms}\) and half-spread \[
\rho^{\alpha_i}a_i(m,z)z+\text{higher terms},
\] where \(a_i\) is a nonzero homogeneous polynomial. Note that all pairs stay on the original affine line. At a nonskew next score, the constant-in-mean part of the contrast signature gives outgoing order \(3\alpha_i\) and a nonzero coefficient. At a skew next score, \eqref{eq:app-b2} has bilinear argument of order \(\rho^{\alpha_i+1}\), and hence outgoing order \[
\alpha_i+\ell(\alpha_i+1)=(\ell+1)\alpha_i+\ell,
\] again with a nonzero polynomial coefficient. The mean increment is always of order at least \(2\alpha_i+1\ge3\), so the induction's leading mean remains \(\rho m\).

After a known prefix entering layer \(j\) with half-spread order \(\alpha\), the first mean coefficient involving that score is, at degree \(2\alpha+1\), \[
g_j(0)a_{j-1}(m,z)^2Lz\,(m^\top M_jz).
\tag{C.6}\label{eq:app-c6}
\] With the prefix fixed, $a_{j-1}$ is known and nonzero. Taking a nonzero row of $L$ and canceling $a_{j-1}^2\ell z$ therefore recovers the bilinear form $m^\top M_jz$, hence $M_j$. To see why later layers cannot change this coefficient, observe that every competitor with the same prefix satisfies
\[
 \operatorname{ord}_\rho\widetilde\Delta^{(j)}\ge3\alpha
 \quad\Longrightarrow\quad
 \operatorname{ord}_\rho(\widetilde\mu_{j+1}-\widetilde\mu_j)
 \ge6\alpha+1>2\alpha+1.
\]
Thus the suffix cannot affect the extracted coefficient, even if the competing next score has a different skewness type. Starting with the readout recovered at consensus, we apply this step until the first collapsing score is identified, or until all scores are recovered. Since the contrast orders strictly increase, the largest degree used is $2\alpha_{s-1}+1$.

For sharpness, fix the prefix through $s-1$ and perturb $M_s$ by a small nonzero matrix. With $N=2\alpha_{s-1}+1$, the mean bound and analytic suffix give
\[
 T_{N-1}(\widetilde\eta)=T_{N-1}(\eta),\qquad
 T_N(\widetilde\eta)\ne T_N(\eta),
\]
where the second inequality is the nonzero change in \eqref{eq:app-c6}. Thus the lower jet has a strictly larger fiber.

For immersion without a collapsing prefix, set earlier variations to zero and differentiate \eqref{eq:app-c6}. The current-score term is injective. To compare it with later contributions, note that a later mean increment is quadratic in the next half-spread. The original contrast and its variation have orders at least $\alpha_j\ge3\alpha$ and $3\alpha$, respectively, so their product with the order-one mean has order
\[
 \alpha_j+3\alpha+1>2\alpha+1.
\]
The differential is therefore triangular and injective. For $\ell=2$, the recurrence gives $\alpha_i=3^i+v_i$, the stated nonzero-curvature specialization.

\subsection{Exact global sign discrimination}\label{exact-global-sign-discrimination}

\begin{proposition}[Exact complete-fiber order on the regular locus]\label{prop-sign-sharp}

Assume \(L\ne0\) and \(\operatorname{sym}M_i\ne0\) for every layer. With \(\alpha_i,k_i\) from Proposition~\ref{prop-heterogeneous-jets}, let \(o_i\) be the first nonzero odd degree of \(h_i\), when one exists, and put \[
\varsigma_i=\begin{cases}o_i,&h_i\text{ is not even},\\ k_i,&h_i\text{ is even},\end{cases}
\qquad N_{\mathrm{sign}}=\max_i\{2\alpha_{i-1}+2\varsigma_i-1\}.
\] Then \(N_{\mathrm{sign}}\) is the least order at which the entire zero-input jet determines the complete effective fiber. The remaining independent signs at even normalizers are exact symmetries at every order. The competing tuple need not initially be regular.

\end{proposition}

We again proceed one layer at a time. Once the earlier layer maps agree, the prefix coefficient is common to both realizations, and the leading score coefficient reduces the next competing score to at most two candidates. On the diagonal $z=m$, equality of $P_k$ gives equality of powers of the quadratic forms. The argument of Section~\ref{a-signature-determines-one-score} then gives
\[
 P_k(M)=P_k(N)\ \Longrightarrow\ (m^\top Mm)^k=(m^\top Nm)^k
 \ \Longrightarrow\ N\in\{M,-M\}.
\]
The sign is fixed as a polynomial identity, and cancellation recovers the full bilinear form with that sign. In particular, both candidates have nonzero symmetric part. It remains to determine the first degree at which $M$ and $-M$ can be distinguished.

If \(h(s)=\sum_{r\ge1}b_rs^r\), then \[
\mathsf H_M-\mathsf H_{-M}
=2\sum_{\substack{r\ge1\\r\ \mathrm{odd}}}
rb_r(m^\top Mm)^{r-1}m^\top Mz.
\tag{C.7}\label{eq:app-c7}
\] If the first odd degree is $o$, the difference starts in mean degree $2o-1$. Its product with the common squared prefix amplitude in \eqref{eq:app-a2} has bidegree
\[
 (\deg_\varepsilon,\deg_m)
 =(2\cdot3^{j-1},\ 2v_{j-1}+2o-1),
 \qquad \deg_X=2\alpha_{j-1}+2o-1.
\]
We first extract the displayed spread degree, which removes every later-layer contribution, and then its lowest mean degree. Since \eqref{eq:app-c7} has no lower mean-degree term, the resulting coefficient is its leading term multiplied by $A_{j-1}^2$. This is an identity of homogeneous polynomials, so the nonzero prefix factor can be canceled exactly. If there is no odd term, $\phi$ is even and either sign gives the same layer map. Otherwise the displayed coefficient resolves the sign. In both cases the recovered layer maps agree, so we can continue to the next layer. Induction therefore gives
\[
 T_{N_{\mathrm{sign}}}^{-1}(T_{N_{\mathrm{sign}}}(\eta))
 =\mathcal R^{-1}(\mathcal R(\eta))
\]
against arbitrary competitors, proving sufficiency in \hyperref[prop-sign-sharp]{the exact sign proposition}.

For necessity, first suppose a maximizing index \(j\) has $h_j$ that is not even. Flip only \(M_j\). At arbitrary inputs \(X=\rho Y\), the entering cloud has mean \(O(\rho)\) and centered part \(O(\rho^\alpha)\), \(\alpha=\alpha_{j-1}\). With scores split as in Section~\ref{an-order-estimate-for-arbitrary-clouds}, the difference of row-normalized log weights is \[
2\{h_{j,\mathrm{odd}}(\bar s_a+\widetilde s_{ab})
-h_{j,\mathrm{odd}}(\bar s_a)\}
=O(\rho^{\alpha+2o_j-1}).
\] Both normalized rows sum to one, so multiplication uses only centered values and gives a layer-output difference $O(\rho^{2\alpha+2o_j-1})$. The fixed analytic suffix is locally Lipschitz and cannot lower its order. Equation~\eqref{eq:app-c7} supplies a nonzero coefficient at that degree. Hence the sign-flipped tuple satisfies
\[
 T_{N_{\mathrm{sign}}-1}(\widetilde\eta)=T_{N_{\mathrm{sign}}-1}(\eta),
 \qquad \mathcal R(\widetilde\eta)\ne\mathcal R(\eta).
\]

If every maximizing index has even $h_i$, then $\varsigma_i=k_i$ there. Strict increase of $N_i$ and $N_{\mathrm{sign}}\ge N_D$ force $N_{\mathrm{sign}}=N_D$. A small regular perturbation of $M_D$ outside $\{M_D,-M_D\}$ has the same lower jet by \eqref{eq:app-c1}, but lies outside the realization fiber by Theorem~\ref{thm-identification}. This supplies the lower bound in the remaining case and completes the proof of the optimal Taylor degree. Having determined which coefficients suffice, we turn in Section~\ref{finite-ordinary-evaluations} to observations obtained by evaluating the function on finitely many inputs.

\subsection{Multidirectional covariance recovery}\label{covariance-expansion}

\begin{proposition}[Recovery through invertible prefixes]\label{prop-covariance}

Suppose \(n\ge d+1\), \(g_i(0)\ne0\) at every layer, \(L\ne0\), and \(M_i\) is invertible for \(i<D\). The final matrix is arbitrary. Then the complete effective fiber is a singleton, and the jet through \(2\cdot3^{D-1}+1\) determines it against arbitrary competitors and is an immersion. No lower jet depends on the final score. In particular, invertible skew prefixes are permitted, and the readout \(L\) may have a nontrivial kernel.

\end{proposition}

To prove \hyperref[prop-covariance]{the invertible-prefix proposition}, replace the one-dimensional contrast amplitude by a covariance matrix. Start with \(x_a=\rho(m+y_a)\), \(\sum_a y_a=0\), and \(C_0=n^{-1}\sum_a y_ay_a^\top\). For an incoming mean \(\mu=O(\rho)\), contrast \(\Delta_a=O(\rho^\alpha)\), and \(g(0)\ne0\), subtract the row-constant log weight:
\[
h(x_a^\top M(\mu+\Delta_b))-h(x_a^\top M\mu)
=g(0)x_a^\top M\Delta_b+O(\rho^{\alpha+3}).
\] The quadratic remainder is sufficiently high because \(2\alpha+2\ge\alpha+3\). Normalization and multiplication by centered values give, with \(C=n^{-1}\sum_b\Delta_b\Delta_b^\top\), \[
x_a'=\mu+g(0)CM^\top x_a+O(\rho^{2\alpha+3}),
\quad
\Delta_a'=g(0)CM^\top\Delta_a+O(\rho^{3\alpha+2}).
\tag{C.8}\label{eq:app-c8}
\] The centered remainder uses the monomial argument of \eqref{eq:app-c1}: query dependence replaces an order-one mean by an order-$\alpha$ contrast and adds $\alpha-1$ powers. Thus its order is
\[
 (2\alpha+3)+(\alpha-1)=3\alpha+2.
\]
Thus centering retains the additional $\alpha-1$ powers supplied by the query dependence.

The leading centered degree triples at each layer. After removing the factor $\rho^{3^i}$ from its centered tokens, let $C_i$ be the covariance of their leading coefficients. The centered expansion then gives \[
C_i=g_i(0)^2C_{i-1}M_i^\top C_{i-1}M_iC_{i-1}.
\tag{C.9}\label{eq:app-c9}
\] Congruence by an invertible prefix gives $C_0>0\Longrightarrow C_i>0$, so no leading covariance vanishes. Exact prefix subtraction isolates, at $N_j=2\cdot3^{j-1}+1$, the coefficient
\[
g_j(0)L C_{j-1}(C_0)M_j^\top m.
\tag{C.10}\label{eq:app-c10}
\]

To recover $M_j$ from \eqref{eq:app-c10} when $L$ has a kernel, we vary the covariance entering the layer. We will show that it can be any positive-definite matrix. For any invertible \(M\), the map \(\Phi_M(S)=SM^\top SMS\) is onto the positive-definite cone. Given \(T>0\), substitute \(S=T^{1/2}XT^{1/2}\) and \(B=T^{1/2}MT^{1/2}\). Then \[
\Phi_M(S)=T
\quad\Longleftrightarrow\quad
XB^\top XBX=I
\quad\Longleftrightarrow\quad
X=(B^\top XB)^{-1/2}.
\tag{C.11}\label{eq:app-c11}
\] This is the one-term exponent \(-1/2\) case of Duan, Peng and Duan's fixed-point equation \citep[Theorem 5]{duan2009}. A self-contained proof uses the Thompson distance
\[
d_T(X,Y)=\inf\{a\ge0:e^{-a}X\preceq Y\preceq e^aX\}.
\] If $e^{-a}X\preceq Y\preceq e^aX$, congruence by $B$ preserves these inequalities. Inversion reverses the inequalities, and taking the positive square root replaces $e^{\pm a}$ by $e^{\pm a/2}$. Thus
\[
 G(X)=(B^\top XB)^{-1/2},\qquad d_T(G(X),G(Y))\le\tfrac12d_T(X,Y).
\]
Starting from $X_0=I$, define $X_{k+1}=G(X_k)$. The contraction estimate makes successive distances summable, so the iterates are Cauchy and stay at bounded Thompson distance from $I$. They therefore lie in an order interval $aI\preceq X_k\preceq bI$, with $0<a\le b<\infty$. Since this interval is compact and bounded away from singular matrices, the iterates converge to a positive-definite limit. Continuity makes this limit a fixed point, and contraction gives uniqueness. Substituting the fixed point into Equation~\eqref{eq:app-c11} proves surjectivity. Finally, multiplication by $g_i(0)^2>0$ preserves the positive-definite cone, so every iterated prefix covariance map in \eqref{eq:app-c9} is onto.

When \(n\ge d+1\), every \(C_0>0\) is attained by a centered cloud: choose \(U\in\mathbb R^{d\times n}\) with \(U\mathbf1=0\), \(UU^\top=nI\), and take its transformed columns \(Y=C_0^{1/2}U\). The mean \(m\) is independent. For a nonzero row \(l\) of \(L\), if \[
lC_{j-1}(C_0)T^\top m=0\quad\text{for all }C_0>0,m,
\] then surjectivity lets us vary the prefix covariance over every positive-definite $C$. Since this cone is open in the space of symmetric matrices, linearity extends the identity to that whole space:
\[
 TCl^\top=0\quad(C>0)
 \ \Longrightarrow\ TCl^\top=0\quad(C=C^\top)
 \ \Longrightarrow\ T=0,
\]
because every $v\in\mathbb R^d$ equals $Cl^\top$ for some symmetric $C$. For example, writing $u=l^\top$, take
\[
C=\frac{vu^\top+uv^\top}{u^\top u}-\frac{u^\top v}{(u^\top u)^2}uu^\top,
\qquad Cu=v.
\]
Thus varying the covariance detects every score direction using one nonzero output row.

Consensus fixes $L$. With a common recovered prefix, equality of \eqref{eq:app-c10} gives the displayed covariance identity with $T=M_j-\widetilde M_j$, hence $M_j=\widetilde M_j$. Every competing prefix thus becomes invertible as a conclusion. The order bounds of \eqref{eq:app-c1} exclude the suffix even for a singular next candidate. Induction, including the arbitrary final matrix, yields
\[
 T_{N_D}^{-1}(T_{N_D}(\eta))=\{\eta\},\qquad N_D=2\cdot3^{D-1}+1.
\]
Replacing differences by variations proves $\ker DT_{N_D}(\eta)=0$. Every lower jet is independent of $M_D$ by \eqref{eq:app-c1}, while a small final-score change leaves the singleton fiber. This proves complete-fiber determination, immersion and sharpness.

\subsection{Recovery from generic two-plane restrictions}\label{recovery-from-generic-two-plane-restrictions}

\begin{proposition}[Recovery through nondegenerate two-plane restrictions]\label{prop-two-plane}

Let \(d\ge2\), \(n\ge3\), \(L\ne0\), and \(g_i(0)\ne0\) at every layer. For a real \(d\times2\) matrix \(U\), suppose each polynomial \[
p_i(U)=\det(U^\top M_iU),\qquad i<D,
\] is not identically zero. The final score is arbitrary. Then the complete effective fiber is a singleton against arbitrary competitors. The zero-input jet through \[
N_D=2\cdot3^{D-1}+1
\] determines that fiber and is an immersion, and no lower jet depends on the final score. In particular these conclusions hold whenever every prefix score is nonzero and skew, in both odd and even dimension.

\end{proposition}

\textbf{Proof.} A two-plane frame is a matrix $U\in\mathbb R^{d\times2}$ of rank two, whose columns form a basis of the plane on which we restrict the inputs. We need one frame for which the readout is nonzero and every compressed prefix is invertible. Define $\mathcal U$ to be the nonvanishing set of the polynomial
\[
\det(U^\top U)\,\|LU\|_F^2\prod_{i<D}\det(U^\top M_iU).
\]
Every factor is a nonzero real polynomial, including $\|LU\|_F^2$ since $L\ne0$. Their product is nonzero, so $\mathcal U$ is nonempty and Euclidean open. For every $U\in\mathcal U$,
\[
 \operatorname{rank}U=2,\qquad LU\ne0,\qquad
 \det(U^\top M_iU)\ne0\quad(i<D).
\]
Thus one common family of planes works for the whole marked prefix.

For every fixed full-rank \(U\) and every two-dimensional token array \(Z\), the exact restriction identity is \[
F_{L,\mathbf M}(UZ)
=F_{LU,(U^\top M_iU)_{i=1}^D}(Z).
\tag{C.12}\label{eq:app-c12}
\] Each score on the left is \(z_a^\top U^\top M_iUz_b\), and a normalized weighted combination of tokens in \(\operatorname{im}U\) remains there. These facts prove \eqref{eq:app-c12} layer by layer. The identity holds for every competing tuple as well.

Equality of input jets through \(N_D\) remains true after the linear substitution \(X=UZ\). For each \(U\in\mathcal U\), apply Proposition~\ref{prop-covariance} in width two: \(n\ge3=2+1\), the readout \(LU\) is nonzero, and the marked compressed prefixes are invertible. Its conclusion against arbitrary competitors gives \[
(L-L')U=0,\qquad U^\top(M_i-M_i')U=0\quad(1\le i\le D).
\tag{C.13}\label{eq:app-c13}
\] The compressed competing tuple is identified by that proposition without any initial invertibility assumption on its prefix. Each entry in \eqref{eq:app-c13} is a polynomial in $U$. Since it vanishes on the nonempty open set $\mathcal U$, it vanishes for every $U$. With $U=[a\ b]$, its off-diagonal entry gives
\[
 a^\top(M_i-M_i')b=0\quad\text{for all }a,b
 \quad\Longrightarrow\quad M_i=M_i',
\]
and $(L-L')U=0$ gives $L=L'$. Thus the jet fiber is the singleton realization fiber.

For immersion, linear restriction commutes with the differential:
\[
 DT_{N_D}(\eta)[\dot\eta]=0
 \ \Longrightarrow\ \delta L\,U=0,\quad U^\top\delta M_iU=0
 \quad(U\in\mathcal U).
\]
The compressed immersion from Proposition~\ref{prop-covariance} supplies the implication. Since the resulting identities are polynomial in $U$, vanishing on $\mathcal U$ gives $\dot\eta=0$. Every lower jet is independent of the final score by Section~\ref{an-order-estimate-for-arbitrary-clouds}, whereas a small nonzero final-score change leaves the singleton fiber. The degree is therefore necessary and sufficient.

Finally, for a nonzero skew \(M\), \[
U^\top MU=\begin{pmatrix}0&a^\top Mb\\-a^\top Mb&0\end{pmatrix},
\qquad \det(U^\top MU)=(a^\top Mb)^2\not\equiv0.
\] This verifies the condition for every nonzero skew prefix, including singular skew matrices in odd dimension. \(\square\)

The restriction condition includes ambient singular skew scores because their generic two-plane compressions are invertible. A rank-one score has singular two-plane compressions, so it instead falls under Theorem~\ref{thm-identification} when its symmetric part is nonzero. Deeper stacks mixing these two types require a recovery argument that combines the two kinds of probes.

\subsection{Proof of Corollary~\ref{cor-depth-two}}\label{complete-depth-two-classification-with-three-or-more-tokens}


\textbf{Proof.} For Corollary~\ref{cor-depth-two}(\ref{cor-depth-two:nonzero-prefix}), the alternatives for $M_1\ne0$ are exhaustive:
\[
 \begin{array}{c|c}
 \operatorname{sym}M_1\ne0&\text{nonzero-slope recovery, Section~\ref{nonzero-slope-rational-local-inversion-and-uniform-determination}}\\
 M_1^\top=-M_1&d\ge2,\ \text{two-plane recovery, Proposition~\ref{prop-two-plane}}.
 \end{array}
\]
Both give a singleton jet fiber and injective differential at degree seven, with arbitrary final matrix.

For Corollary~\ref{cor-depth-two}(\ref{cor-depth-two:zero-prefix}), \(M_1=0\), so the first layer produces consensus and the second score has no effect. Consensus inputs determine \(L\). The degree-three mean coefficient involving the first score is \[
g_1(0)Lz\,(m^\top M_1z)
\] on the line probes. For $L\ne0$, polynomial cancellation makes it injective in $M_1$. Moreover, a competing second layer contributes only from degree seven. Hence
\[
 T_3^{-1}(T_3(L,0,M_2))
 =\mathcal R^{-1}(\mathcal R(L,0,M_2))
 =\{(L,0,N):N\in\mathbb R^{d\times d}\}.
\]
Every degree-two jet is score-independent, proving necessity of degree three. Consensus gives the zero-readout case in Corollary~\ref{cor-depth-two}(\ref{cor-depth-two:zero-readout}). \(\square\)

\subsection{Finite ordinary evaluations}\label{finite-ordinary-evaluations}

Let $V=\operatorname{span}\{f_1,\ldots,f_m\}$ be a space of linearly independent vector-valued analytic functions on a connected open input neighborhood. We seek finitely many inputs at which their evaluations remain linearly independent. For an input $X$ and output component $a$, write $e_{X,a}(f)=f(X)_a$.

Start with $K_0=V$. If $K_j\ne0$, choose $f\in K_j$ that is not identically zero, and choose $X_j,a_j$ with $f(X_j)_{a_j}\ne0$. Then
\[
K_{j+1}=K_j\cap\ker e_{X_j,a_j}\subsetneq K_j.
\]
Each step lowers the dimension, so after at most $m$ steps we obtain $K_q=0$. Taking all output components at the selected inputs gives a bank $\mathcal B$ with
\[
\ker(\operatorname{ev}_{\mathcal B}|_V)=0.
\]
A maximal minor of this evaluation matrix is nonzero at the constructed bank and is analytic in the bank points. Its zero set has Lebesgue measure zero. Consequently, sufficiently many independent inputs drawn with a density on the full-dimensional neighborhood separate this fixed function space almost surely. Both the bank size needed for a particular construction and its conditioning concern this fixed space. In contrast, inputs supported on consensus or on a discrete corpus may fail to separate it.

\section{Native contact spectra}\label{appendix-d}

For the analytic path $t\mapsto\theta(t)$ of Theorem~\ref{thm-spectrum} and a finite input bank $\mathcal B$, set
\[
 \Psi_{\mathcal B}=\operatorname{ev}_{\mathcal B}\circ\mathcal R,
 \qquad \mathcal O_{\mathcal B}=\Psi_{\mathcal B}\circ\pi,
 \qquad J(t)=D_\theta\mathcal O_{\mathcal B}|_{\theta(t)}.
\]
Here $J(t)$ differentiates with respect to every native parameter, while $t$ only specifies the point where that differential is evaluated. Let $\mathcal O=\mathbb R\{t\}$ be the ring of convergent power series near $t=0$. We seek row and column transformations, analytic and invertible at zero, which reduce $J(t)$ to a rectangular diagonal matrix with entries $1$, $t^{\beta_i}$, and zero. These powers are its Smith factors. A positive power records a direction whose sensitivity tends to zero. A zero column instead records a direction whose differential vanishes identically along the path.

\paragraph{Proof overview.}
The proof has two steps. First, row normalization cancels the row-constant part of the log weights. Together with homogeneity, this makes the centered output after layer $i$ divisible by $t^{e_i}$ and the layer-$i$ factor derivatives divisible by $t^{\beta_i}$. For common first nonconstant degree $k$,
\[
 e_0=0,\qquad e_i=3e_{i-1}+2k,\qquad
 \beta_i=2e_{i-1}+2k-1.
\]
The layer mean first changes at order $2e_{i-1}+2k$ in $t$. Differentiating with respect to one query or key factor reduces this order by one, giving $\beta_i$. Section~\ref{homogeneity-gives-joint-analytic-divisibility} proves these divisibilities jointly in the parameter $t$ and the input.

Second, restrict the Jacobian to a complement of its exact source kernel and divide each layer-$i$ derivative column by $t^{\beta_i}$. Collect these divided columns in a matrix $E(t)$, which we call the divided derivative frame. We will show that its columns remain independent at $t=0$, giving
\[
 \begin{aligned}
 J_{\mathrm{vis}}(t)&=E(t)D(t),\\
 D(t)&=\operatorname{diag}\bigl(I_{d^2},(t^{\beta_i}I_{m_i})_{i=1}^D\bigr),\\
 m_i&=r_i(2d-r_i),\qquad
 \operatorname{rank}E(0)=d^2+\sum_{i=1}^D m_i.
 \end{aligned}
\]
Triangular score signatures establish independence of the leading derivative functions. Then Section~\ref{finite-ordinary-evaluations} supplies a finite input bank preserving that independence. A maximal minor of $E(t)$ is therefore nonzero at $t=0$, hence a unit in $\mathcal O$. Analytic row reduction gives the exact diagonal powers in $D(t)$, so the divisibility bounds are attained. Sections~\ref{the-divided-frame-and-exact-smith-factors}--\ref{factor-kernels-multiplicities-and-singular-values} carry out this reduction, count the exact kernel, and prove matching singular-value bounds. The construction establishes existence of a separating bank. Optimizing the number and placement of its inputs remains a separate problem.

\subsection{The native source and its value coordinates}\label{the-native-source-and-its-value-coordinates}

\begin{proposition}[Heterogeneous native contact orders]\label{prop-heterogeneous-spectrum}

Under the native value, factor-rank and symmetric-part hypotheses of Theorem~\ref{thm-spectrum}, permit the first nonconstant normalizer degrees \(k_i\) to differ. Define \[
e_0=0,\qquad e_i=3e_{i-1}+2k_i,\qquad
\beta_i=2k_i-1+2e_{i-1}
=2k_i-1+4\sum_{j<i}3^{i-1-j}k_j.
\] On a fixed bank separating the divided leading derivative frame, the Smith factors over \(\mathbb R\{t\}\) consist of \(d^2\) units and \(r_i(2d-r_i)\) factors \(t^{\beta_i}\) for each layer, with \((D-1)d^2+\sum_i r_i^2\) exact source-kernel directions. Each unit factor corresponds to a singular value comparable to \(1\), while the other positive singular values are comparable to their powers of \(|t|\). Such finite banks exist, and sufficiently many independent inputs drawn with a density on a full-dimensional open neighborhood separate the fixed leading frame almost surely.

\end{proposition}

Near \(V_i=I_d\), put \[
P_i=V_i\cdots V_1\ (i<D),\quad L=P_D,
\qquad\widetilde Q_i=P_{i-1}^\top Q_i,\quad
\widetilde K_i=P_{i-1}^\top K_i.
\] The inverse source coordinates are
\[
 V_i=P_iP_{i-1}^{-1},\qquad P_D=L,\qquad
 Q_i=P_{i-1}^{-\top}\widetilde Q_i,\qquad
 K_i=P_{i-1}^{-\top}\widetilde K_i.
\]
These formulas give an analytic change of native coordinates near $V_i=I_d$, including at $t=0$. In the new coordinates, $M_i=\widetilde Q_i\widetilde K_i^\top$ and the realized function depends only on $L$ and the $M_i$. The matrices $P_1,\ldots,P_{D-1}$ are independent coordinates absent from that function, contributing $(D-1)d^2$ exact kernel directions. The path becomes
\[
 \widetilde Q_i=tU_i,\qquad\widetilde K_i=tW_i,\qquad L=I.
\]
The Jacobian of this coordinate change and its inverse are analytic at zero. They therefore preserve the Smith factors and change singular values by at most fixed multiplicative constants near zero.

\subsection{Homogeneity gives joint analytic divisibility}\label{homogeneity-gives-joint-analytic-divisibility}

For any known normalizer and \(t\ne0\), bilinearity gives \[
S_{t^2B}^{\phi}(X)=t^{-1}S_B^{\phi}(tX).
\] Iterating with identity values gives the exact identity \[
F_{I,t^2\mathbf B}(X)=t^{-1}F_{I,\mathbf B}(tX).
\tag{D.1}\label{eq:app-d1}
\] The right side extends analytically to $t=0$ because $F_{I,\mathbf B}(0)=0$. Thus \eqref{eq:app-d1} converts small score matrices into small inputs, where the order estimates of Appendix~\ref{appendix-c} apply. For softmax, the score multiplier $t^2$ also corresponds to temperature $t^{-2}$.

Let \(B_i=U_iW_i^\top\). At fixed external \(t\), a factor variation produces \[
\delta M_i=tT_i,\qquad
T_i=\delta\widetilde Q_iW_i^\top+U_i\delta\widetilde K_i^\top.
\] Since $M_i=t^2B_i$, this variation corresponds, for fixed $t\ne0$, to $\delta B_i=t^{-1}T_i$. The additional factor $t^{-1}$ in \eqref{eq:app-d1} therefore gives \[
\delta_i F(X)=t^{-2}D_{B_i}F_{I,\mathbf B}(tX)[T_i].
\tag{D.2}\label{eq:app-d2}
\] Write the analytic input expansion as $D_{B_i}F(X)[T_i]=\sum_{\nu\ge N_i}f_{i,\nu}(X;T_i)$, with each $f_{i,\nu}$ homogeneous of degree $\nu$. Section~\ref{an-order-estimate-for-arbitrary-clouds} gives $N_i\ge3$, hence
\[
 \delta_iF(t,X)=\sum_{\nu\ge N_i}t^{\nu-2}f_{i,\nu}(X;T_i)
 \in t^{N_i-2}\mathbb R\{t,X\}^{d\times n}.
\]
After division by $t^{N_i-2}$, the derivative remains analytic jointly in $(t,X)$. Thus the bound applies to every output component on every sufficiently small input, before choosing a finite bank. The argument also allows the normalizer degrees to differ across layers. Since $\alpha_i=e_i+1$, the input recurrence gives
\[
N_i-2=2k_i-1+2e_{i-1}=\beta_i,
\quad e_0=0,\quad e_i=3e_{i-1}+2k_i.
\] At $t=0$ all scores vanish, so the effective stack sends $X$ to its consensus mean $\bar x\mathbf1^\top$. Varying the independent readout coordinate therefore gives $(\delta L\bar x)\mathbf1^\top$ at the endpoint. These $d^2$ directions have order zero.

\subsection{The divided frame and exact Smith factors}\label{the-divided-frame-and-exact-smith-factors}

We now show that the derivative bounds are attained. Divide each layer-$i$ factor derivative by $t^{\beta_i}$. The preceding divisibility makes the quotient analytic, so we can evaluate it at $t=0$. By \eqref{eq:app-d2}, the resulting function is the homogeneous input-degree-\(N_i\) part of \(D_{B_i}F_{I,\mathbf B}(X)[T_i]\). Its mean, restricted to the line probes and then to the appropriate mean degree, is a nonzero scalar times \[
A_{i-1}^2 z\,DP_{k_i}(B_i)[T_i].
\tag{D.3}\label{eq:app-d3}
\] Section~\ref{immersion-and-the-sharp-local-order} gives injectivity in $T_i$ after canceling the nonzero prefix amplitude and a nonzero component of $z$. Each group is therefore independent on a complement to its factor-tangent kernel. Across groups, homogeneous degree separates them:
\[
 1<N_1<\cdots<N_D,\qquad
 f_0+\sum_i f_i=0,\quad \deg f_0=1,\ \deg f_i=N_i
 \quad\Longrightarrow\quad f_0=f_1=\cdots=f_D=0.
\]
Here $f_0=\delta L\bar x$ is the value-product endpoint function. This proves joint independence, including the absence of interlayer cancellations.

Since the leading derivative functions are independent, Section~\ref{finite-ordinary-evaluations} supplies a finite bank on which their evaluations remain independent. Fix this bank and evaluate the analytic divided derivatives to obtain $E(t)$. Its value $E(0)$ has full column rank, so we can select a square row submatrix whose determinant is nonzero at zero and hence a unit in $\mathcal O$. After permuting rows, write $E=(E_1^\top,E_2^\top)^\top$ with $E_1(0)$ invertible. The explicit analytic row operation is
\[
 U(t)=\begin{pmatrix}E_1^{-1}&0\\-E_2E_1^{-1}&I\end{pmatrix},
 \qquad U(t)E(t)=\begin{pmatrix}I\\0\end{pmatrix}.
\]
Applying this row operation to
\[
J_{\mathrm{vis}}(t)
=E(t)\operatorname{diag}\bigl(I_{d^2},
t^{\beta_1}I_{m_1},\ldots,t^{\beta_D}I_{m_D}\bigr),
\]
gives the diagonal blocks $I_{d^2},t^{\beta_1}I_{m_1},\ldots,t^{\beta_D}I_{m_D}$, with zero rows below them. Thus no further powers of $t$ remain hidden in $E(t)$: the displayed exponents are the exact Smith orders. Since the selected determinant stays nonzero near zero, the same bank works for every sufficiently small $t$.

\subsection{Factor kernels, multiplicities and singular values}\label{factor-kernels-multiplicities-and-singular-values}

For full-column-rank \(U,W\in\mathbb R^{d\times r}\), the constant map \[
(\Delta Q,\Delta K)\longmapsto\Delta QW^\top+U\Delta K^\top
\]
has kernel $\{(UA,-WA^\top):A\in\mathbb R^{r\times r}\}$. Substitution shows that each displayed pair lies in the kernel. Conversely, let $q_U:\mathbb R^d\to\mathbb R^d/\operatorname{im}U$ be the quotient map and apply it columnwise to a kernel pair. Then
\[
 \Delta QW^\top+U\Delta K^\top=0
 \ \Longrightarrow\ q_U(\Delta Q)W^\top=0
 \ \Longrightarrow\ \Delta Q=UA
 \ \Longrightarrow\ \Delta K=-WA^\top.
\]
The middle implication uses the full row rank of $W^\top$, and the last uses a left inverse of $U$. These $r^2$ independent directions give columns that vanish identically along the path. Thus
\[
m_i=2dr_i-r_i^2=r_i(2d-r_i).
\] Choose a constant linear complement in each factor pair. Together with the free prefix-value coordinates this yields exactly \[
\zeta=\sum_i r_i^2+(D-1)d^2
\]
zero source columns, with count identity
\[
 d^2+\sum_i m_i+\zeta=Dd^2+2d\sum_i r_i.
\]
These columns vanish for all $t$ along the path. The remaining score columns are nonzero for small $t\ne0$ and vanish only at $t=0$, where the pointwise kernel is consequently larger.

Solving the recurrence gives \[
e_i=2\sum_{j\le i}3^{i-j}k_j,
\qquad
\beta_i=2k_i-1+4\sum_{j<i}3^{i-1-j}k_j.
\] For common \(k\), \(e_i=k(3^i-1)\), giving \(\beta_i=2k3^{i-1}-1\). This completes the algebraic part of Proposition~\ref{prop-heterogeneous-spectrum} and its common-degree specialization, Theorem~\ref{thm-spectrum}.

To pass from the algebraic orders to singular values, write the reduction as $J(t)=A(t)D(t)B(t)$, including the zero rows and columns in $D(t)$. Both $A,B$ and their inverses are analytic at zero, hence have bounded operator norms on a sufficiently small compact interval. For the resulting diagonal matrix $D(t)$, singular-value min--max gives
\[
 c\,\sigma_j(D(t))\le\sigma_j(J(t))\le C\,\sigma_j(D(t)),
 \qquad 0<|t|\ll1,
\]
for every ordered nonzero singular value. This is the convergent analytic invariant-factor comparison \citep{kaveh2022}, where $c,C>0$ depend on the bank, normalizers and path. If the observation metric satisfies $c_0I\preceq\mathcal W\preceq C_0I$, then
\[
 c_0J^\top J\preceq J^\top\mathcal WJ\preceq C_0J^\top J.
\]
Its nonzero eigenvalues have orders $0$ with multiplicity $d^2$, and $2\beta_i$ with multiplicity $m_i$. This gives the spectrum of the Gauss--Newton matrix. For a squared loss, the full Hessian also contains terms weighted by the observation residual, and those terms vanish at an exact fit.

\subsection{Covariance and two-plane extensions}\label{covariance-and-two-plane-extensions}

\begin{corollary}[Native spectrum through invertible prefixes]\label{cor-covariance-spectrum}

In the native setting of Theorem~\ref{thm-spectrum}, instead assume \(n\ge d+1\), \(g_i(0)\ne0\), and \(U_iW_i^\top\) invertible for \(i<D\), so \(r_i=d\) there. The final factors remain full column rank but their product need not have nonzero symmetric part. For a separating fixed finite bank, the same multiplicities and kernel count hold, with \[
\beta_i=2\cdot3^{i-1}-1.
\] No symmetric-part hypothesis is needed on the invertible prefix.

\end{corollary}

Under Corollary~\ref{cor-covariance-spectrum}, every normalizer has first nonconstant degree $k_i=1$. The divisibility proof \eqref{eq:app-d1}--\eqref{eq:app-d2} still gives \(\beta_i=N_i-2=2\cdot3^{i-1}-1\) for every group. The hypothesis \(\operatorname{sym}B_i\ne0\) was used only to prove injectivity of the divided leading functions in Section~\ref{the-divided-frame-and-exact-smith-factors}. Replace that step by \eqref{eq:app-c10}: the mean of the homogeneous degree-\(N_i\) derivative is \[
g_i(0)C_{i-1}(C_0)T_i^\top m.
\] Covariance surjectivity from Section~\ref{covariance-expansion} gives
\[
 g_i(0)C_{i-1}(C_0)T_i^\top m\equiv0\quad\Longrightarrow\quad T_i=0.
\]
Only the preceding matrices must be invertible. For the final score, the argument needs only the stated full-column-rank factors for its tangent rank count. Distinct input degrees give joint independence. The same bank construction, row reduction and factor-kernel calculation then give every power, multiplicity and zero of the full native Jacobian, including the value entries.

\begin{corollary}[Native spectrum through nondegenerate two-plane restrictions]\label{cor-two-plane-spectrum}

Keep the full-column-rank native factor and trainable identity-value hypotheses of Theorem~\ref{thm-spectrum}. Instead of the symmetric-part conditions, assume \(d\ge2\), \(n\ge3\), \(g_i(0)\ne0\), and that each polynomial \[
U\longmapsto\det\!\bigl(U^\top(U_iW_i^\top)U\bigr),\qquad i<D,
\] is nonzero. No symmetric-part condition is imposed on the final factor product. Then a fixed separating finite bank has the same multiplicities and exact kernel count, with \(\beta_i=2\cdot3^{i-1}-1\). In particular this holds when all prefix products are nonzero skew, even if they are singular in the ambient space.

\end{corollary}

\textbf{Proof.} The all-input divisibility in Section~\ref{homogeneity-gives-joint-analytic-divisibility} uses no symmetric-part condition. Restrict the divided derivative associated to $T_i$ to inputs $X=UZ$, with $U$ in the admissible frame set $\mathcal U$ of Proposition~\ref{prop-two-plane}. By \eqref{eq:app-c12}, it is the width-two endpoint for $U^\top T_iU$. Covariance injectivity gives
\[
 \text{ambient endpoint}=0
 \ \Longrightarrow\ U^\top T_iU=0\quad(U\in\mathcal U)
 \ \Longrightarrow\ T_i=0,
\]
the last implication being polynomial vanishing on a nonempty open set. Each endpoint map is therefore injective on its factor-tangent image. Its input degree differs from every other score group and from the degree-one value group, proving joint independence. The finite-bank construction, analytic row reduction and kernel calculation of Sections~\ref{the-divided-frame-and-exact-smith-factors}--\ref{factor-kernels-multiplicities-and-singular-values} finish the full native spectrum. \(\square\)

\section{Geometry of the realization map}\label{appendix-geometry}

\subsection{Germs, finite observations and parameter fibers}\label{sec:geometry-observations}

We now describe the local geometry underlying the identification results. We keep $D\ge1$, $n\ge2$, and the prescribed positive analytic normalizers. Recall that the effective realization map is
\[
\mathcal R:\Theta_{\mathrm{eff}}\longrightarrow
\mathbb R\{X\}^{p\times n},\qquad
(L,\mathbf M)\longmapsto[F_{L,\mathbf M}]_0.
\]
Here $\mathbb R\{X\}$ is the ring of power series convergent near the zero input. Thus two realized germs agree precisely when their functions agree on some neighborhood of $X=0$. We obtain finite-dimensional observations either by retaining Taylor coefficients through degree $N$, denoted by $j_0^N$, or by evaluating on a finite input bank $\mathcal B$, denoted by $\operatorname{ev}_{\mathcal B}$. These give
\[
 \begin{aligned}
 T_N&=j_0^N\circ\mathcal R:\Theta_{\mathrm{eff}}\longrightarrow\mathcal J_N,\\
 \Psi_{\mathcal B}&=\operatorname{ev}_{\mathcal B}\circ\mathcal R:
 \Theta_{\mathrm{eff}}\longrightarrow(\mathbb R^{p\times n})^{\mathcal B},\\
 \mathcal O_{\mathcal B}&=\Psi_{\mathcal B}\circ\pi.
 \end{aligned}
\]
Here $\mathcal J_N$ is the finite-dimensional space of Taylor coefficient arrays through degree $N$. Theorems~\ref{thm-identification} and~\ref{thm-exceptions} describe the fibers of $\mathcal R$. Theorem~\ref{thm-jets} then shows when the finite observation $T_N$ gives local coordinates on the effective space. Finally, Theorem~\ref{thm-spectrum} studies how the differential of the native observation $\mathcal O_{\mathcal B}$ vanishes along the specified parameter path.

To describe a fiber locally, we instead consider analytic functions of the parameters near a fixed tuple $\theta_*$. For a finite analytic observation map $f$, we define its local parameter ring and residual ideal by
\[
 \mathcal O_{\theta_*}=\mathbb R\{\theta-\theta_*\},\qquad
 I_{\theta_*}(f)=(f_j(\theta)-f_j(\theta_*))_j\subset\mathcal O_{\theta_*}.
\]
The generators are the observation residuals, so their common zero germ is $(f^{-1}(f(\theta_*)),\theta_*)$. The ideal also records how these residuals vanish. For example, $(u)$ and $(u^2)$ have the same zero set but different orders of vanishing. We will use coordinate changes in the finite-dimensional parameter and observation spaces to describe this ideal. Thus the construction does not require a manifold structure on the space of realized input germs.

We first note why Taylor observations are polynomial in the parameters. Every intermediate token vanishes at $X=0$, so each normalization denominator has constant term $n\phi_i(0)=n$. Writing the denominator as $n+u(X)$, with $u(0)=0$, we have
\[
 (n+u)^{-1}\equiv n^{-1}\sum_{j=0}^N(-u/n)^j\pmod{(X)^{N+1}}.
\]
Only finitely many terms survive the truncation, and their denominators are powers of the fixed number $n$. Consequently, $T_N$ and $T_N\circ\pi$ are polynomial maps, with the normalizer coefficients treated as constants. Ordinary evaluations need only be analytic. Nevertheless, whenever $DT_N$ is injective, the finite-evaluation construction in Section~\ref{finite-ordinary-evaluations} supplies a bank for which $D\Psi_{\mathcal B}$ is also injective.

\subsection{The boundary of the analytic argument}\label{sec:flat-boundary}

The proofs above use analyticity in two ways: Taylor coefficients detect a nonzero germ, and a nonzero germ can be canceled from a product identity. We give a smooth example to explain why these steps require analyticity. Define
\[
 b(0)=0,\qquad b(s)=e^{-1/s^2}\ (s\ne0),\qquad \phi(s)=e^{s b(s)}.
\]
Off zero, every derivative of $b$ is a polynomial in $1/s$ times $e^{-1/s^2}$ and extends by zero at the origin. Hence $\phi-1$ is smooth and flat, although
\[
 \phi>0,\qquad (s b(s))'=b(s)(1+2/s^2)>0\quad(s\ne0),
\]
so $\phi$ is strictly increasing and injective.

For one layer, two scalar tokens and unit value, composing the flat function with quadratic scores preserves flatness, and the denominator remains nonzero. If $A(X)=\mu(X)\mathbf1^\top$, then
\[
 j_0^N S_c^\phi=j_0^N A\quad\text{for every }N,c,
 \qquad S_c^\phi(r,0)_1=\frac{r\phi(cr^2)}{\phi(cr^2)+1}\quad(r>0).
\]
Since $\phi$ is strictly increasing, the displayed output distinguishes $c=1$ and $c=2$ for every $r>0$, although all their Taylor coefficients agree. Thus the full function contains score information that every finite jet misses. For the separate cancellation step, consider the two smooth germs
\[
 b_+(s)=b(s)\mathbf1_{s>0}\not\equiv0,\qquad
 b_-(s)=b(s)\mathbf1_{s<0}\not\equiv0,\qquad b_+b_-=0.
\]
Both are nonzero arbitrarily close to zero, but their product vanishes. Thus the smooth germ ring admits zero divisors. These examples explain the limits of the Taylor and cancellation arguments, while leaving full-function identification as a separate question.

\subsection{Proof of Proposition~\ref{prop-native-geometry}}\label{sec:geometry-normal-form}

We first choose coordinates in the native parameter space that separate the value maps from the score factorizations. As in Proposition~\ref{prop-native-geometry}, every token width is $d$ and the value matrices are invertible near the marked point. The query and key matrices satisfy $Q_i,K_i\in\mathbb R^{d\times r_i}$ and may be rank deficient. Set
\[
P_0=I_d,\quad P_i=V_i\cdots V_1,\quad L=P_D,
\qquad
\widetilde Q_i=P_{i-1}^\top Q_i,\quad
\widetilde K_i=P_{i-1}^\top K_i.
\]
The matrices $P_1,\ldots,P_{D-1},L$ and the transformed factors form native coordinates. Indeed, the original parameters can be recovered analytically by
\[
 V_i=P_iP_{i-1}^{-1},\qquad P_D=L,\qquad
 Q_i=P_{i-1}^{-\top}\widetilde Q_i,\qquad
 K_i=P_{i-1}^{-\top}\widetilde K_i.
\]
In this chart the reduction is
\[
\pi=(L,\widetilde Q_1\widetilde K_1^\top,\ldots,
\widetilde Q_D\widetilde K_D^\top).
\]
Thus the prefix-value coordinates $P_1,\ldots,P_{D-1}$ are absent from the effective tuple. The native observation is $\Psi\circ\pi$, where $\Psi$ is the finite effective observation in the proposition. We next choose observation coordinates in which $\Psi$ becomes the inclusion of its effective coordinates.

Let $a=(D+1)d^2$ be the dimension of the effective parameter space. Since $D\Psi$ is injective at $\eta_*$, we can select $a$ observed coordinates, indexed by $I$, for which $D\Psi_I(\eta_*)$ is nonsingular. The analytic inverse-function theorem then gives a local inverse $\Psi_I^{-1}$. For an observation vector $y$, define the new coordinates by
\[
y\longmapsto
\left(\Psi_I^{-1}(y_I),\;
y_{I^c}-\Psi_{I^c}(\Psi_I^{-1}(y_I))\right)
\]
The first coordinate recovers the effective tuple from $y_I$, while the second subtracts its predicted remaining observations. The inverse sends $(v,w)$ to $(\Psi_I(v),w+\Psi_{I^c}(v))$, so this is an invertible analytic change of coordinates. In particular, it sends $\Psi(v)$ to $(v,0)$, giving
\[
 \Psi\circ\pi\ \sim_{\mathrm{an}}\ (\pi,0)
 =\bigl(L,(\widetilde Q_i\widetilde K_i^\top)_i,0\bigr).
\]

We also need to check that changing observation coordinates preserves the residual ideal. For an analytic coordinate change $\kappa$, write $f_*=f(\theta_*)$ and factor its difference along the line segment from $f_*$ to $f$. Applying the same factorization to $\kappa^{-1}$ gives
\[
 \kappa(f)-\kappa(f_*)=A(f)(f-f_*),\qquad
 I(\kappa\circ f)\subseteq I(f)\subseteq I(\kappa\circ f).
\]
The two inclusions give equality of the residual ideals. Thus fixing the observed function fixes $L$ and each product $\widetilde Q_i\widetilde K_i^\top$, while leaving the prefix-value coordinates free. The native fiber germ is therefore
\[
(\mathrm{GL}_d(\mathbb R)^{D-1},\mathbf P_*)\times
\prod_{i=1}^D
\bigl(\mu_i^{-1}(M_i^*),
       (\widetilde Q_i^*,\widetilde K_i^*)\bigr),
\qquad \mu_i(Q,K)=QK^\top.
\]
Here $\mu_i^{-1}(M_i^*)$ is the space of factorizations of the fixed score $M_i^*$ near the marked pair. The argument allows deficient factor ranks, so these may be singular multiplication fibers, as studied by \citet{pepinlehalleur2024}. Thus the product describes all nearby native parameters realizing the marked function. To determine which nearby effective scores are realizable, one would instead study the local image of each multiplication map $\mu_i$.

At full-column-rank factors, Section~\ref{factor-kernels-multiplicities-and-singular-values} gives
\[
 \ker D\mu_i=\{(Q_iA,-K_iA^\top):A\in\mathbb R^{r_i\times r_i}\},\qquad
 \operatorname{rank}D\mu_i=r_i(2d-r_i).
\]
The independent readout and score blocks therefore give $\operatorname{rank}D(\Psi\circ\pi)=d^2+\sum_i r_i(2d-r_i)$ under the immersion hypothesis.

On the regular locus, the signs at even normalizers give separate effective branches, and this construction applies locally at each branch where $D\Psi$ is injective. At collapse, later scores become invisible and $D\Psi$ can acquire a kernel. We can then no longer construct $\Psi_I^{-1}$ by this argument. This is why Theorem~\ref{thm-spectrum} analyzes the derivative through the collapsed endpoint directly. For the fiber at a first zero score, the following argument instead separates the visible prefix from the free suffix.

\subsection{The residual ideal at the first zero score}\label{sec:geometry-residual-ideal}

\begin{proposition}[First-zero residual ideal]\label{prop-collapse-ideal}
Suppose $g_i(0)\ne0$ for every layer, $L_*\ne0$,
$\operatorname{sym}M_i^*\ne0$ for $i<s$, and $M_s^*=0$.
The suffix scores are arbitrary. There is a fixed finite ordinary-input bank whose observation residual ideal at the marked effective tuple is exactly
\[
\bigl(L-L_*,\;M_1-M_1^*,\ldots,M_{s-1}-M_{s-1}^*,\;M_s\bigr).
\]
Its squared residual norm is locally comparable to the sum of the squares of these entries, uniformly over a sufficiently small suffix neighborhood. For a native realization with square invertible values, its pullback is the corresponding prefix multiplication-fiber ideals and the zero-product ideal at layer $s$, with free suffix factors and value-prefix coordinates.
\end{proposition}

\begin{proof}
We separate the parameters that determine the marked function from those erased by the first zero score. Write
\[
 u=(L-L_*,(M_i-M_i^*)_{i<s},M_s),\qquad v=(M_{s+1},\ldots,M_D).
\]
When $u=0$, layer $s$ averages its inputs and every later layer fixes the resulting consensus. Thus the realized function is independent of $v$ on $u=0$. For any finite bank, let $r(u,v)$ be its observation vector minus the marked observation. Since $r(0,v)=0$, integrating the derivative in $u$ gives
\[
r(u,v)=H(u,v)u,\qquad
H(u,v)=\int_0^1 D_ur(\tau u,v)\,d\tau.
\]
Thus every residual is an analytic linear combination of the coordinates of $u$, and the residual ideal satisfies $I=(r_j)_j\subseteq(u)$. To prove the reverse inclusion, we construct a bank whose differential detects every $u$ direction.

For injectivity of the visible differential, consensus first forces $\delta L=0$. With earlier variations zero, Section~\ref{nonzero-slope-rational-local-inversion-and-uniform-determination} gives a nonzero prefix factor times
\[
 g_j(0)L_*z(m^\top\delta M_jz)=0\quad\Longrightarrow\quad\delta M_j=0
 \qquad(j\le s).
\]
The preceding symmetric parts make the prefix factors nonzero. The step $j=s$ survives every suffix because, at $Y=m\mathbf1^\top$,
\[
 D_XS_M^\phi(Y)[\dot Y]=\mu(\dot Y)\mathbf1^\top,\qquad
 D_MS_M^\phi(Y)[\dot M]=0.
\]
Thus the suffix passes the common-mean derivative unchanged, proving visible injectivity through the first zero regardless of $v$.

By Section~\ref{finite-ordinary-evaluations}, we may choose a bank separating these visible derivatives at $(0,v_*)$. Since $H(0,v_*)=D_ur(0,v_*)$, some square row submatrix $H_I$ has nonzero determinant there. Its determinant stays nonzero on a small neighborhood, so $H_I$ has an analytic inverse and
\[
 u=H_I(u,v)^{-1}r_I(u,v)
 \quad\Longrightarrow\quad (u)\subseteq I\subseteq(u).
\]
On a smaller neighborhood, bounded $H$ and $H_I^{-1}$ give
\[
 c\|u\|^2\le\|r(u,v)\|^2\le C\|u\|^2,\qquad c,C>0,
\]
uniformly over the suffix coordinate $v$.

For square invertible native values, the preceding source coordinates pull the ideal back to
\[
\bigl(L-L_*,\;
\widetilde Q_i\widetilde K_i^\top-M_i^*\ (i<s),\;
\widetilde Q_s\widetilde K_s^\top\bigr).
\]
Thus the suffix factors and the prefix-value coordinates are free in this native fiber. The assumed invertibility of the value maps is what makes them independent coordinates.
\end{proof}

The ideal describes the fiber through $M_s=0$. If we move to $M_s\ne0$, layer $s$ may transmit contrast again, allowing the nearby function to depend on the suffix.

\begin{corollary}[Serial collapse is not a product of layer fibers]\label{cor-serial-components}
Fix identity values and softmax, with one-column query/key factors $M_i=q_ik_i^\top$, where $q_i,k_i\in\mathbb R^d$. The scores may have rank zero. At the native point where every factor vanishes, the fiber over uniform averaging has residual ideal $(q_1k_1^\top)$, with every later factor free. It has two irreducible analytic components. For $D\ge2$, this differs from the product of the individual zero-product fibers, which has $2^D$ components.
\end{corollary}

\begin{proof}
Apply Proposition~\ref{prop-collapse-ideal} with $s=1$ and fixed $L=I_d$. In the convergent power-series ring,
\[
(q_1k_1^\top)=(q_1)\cap(k_1).
\]
Every monomial of a series in the intersection contains a $q_1$ coordinate and a $k_1$ coordinate. Grouping the monomials by a choice of such a pair expresses the series as a sum of the generators $q_{1,a}k_{1,b}$ times convergent series on a smaller polydisc. The ideals $(q_1)$ and $(k_1)$ are prime and incomparable, so
\[
 V(q_1k_1^\top)=V(q_1)\cup V(k_1)
\]
has exactly two irreducible analytic components, with all later factors free. For independent zero-product fibers at every layer, the choices $(q_i=0)$ or $(k_i=0)$ are independent, giving $2^D$ distinct coordinate components. The components here are irreducible analytic germs and can intersect at the marked point.
\end{proof}

\paragraph{A local learning-coefficient consequence.}\label{remark:collapse-learning}
We can also use this residual ideal to compute how a local loss integral decays. This gives the local learning coefficient from singular learning theory \citep{watanabe2009,lau2025llc}. Fix $n\ge2$, $D\ge1$, identity values, softmax, and one-column query/key factorizations $M_i=q_ik_i^\top$ with $q_i,k_i\in\mathbb R^d$. Mark the point where every factor is zero. Let $R(\theta)$ be the squared Euclidean observation residual over uniform averaging, using a finite ordinary-input bank supplied by Proposition~\ref{prop-collapse-ideal}. For a sufficiently small fixed product neighborhood $U$ of the marked point and an analytic prior density $\rho$ positive there, define the local pair $(\lambda,\nu)$ by
\[
 Z_U(T)=\int_U e^{-T R(\theta)}\rho(\theta)\,d\theta
 \asymp T^{-\lambda}(\log T)^{\nu-1},
 \qquad T\longrightarrow\infty .
\]
Compare this serial observation with observing the $D$ layers independently, using the corresponding one-layer banks and the sum of their squared residuals. Then
\[
 (\lambda,\nu)_{\mathrm{serial}}=(d/2,2),
 \qquad
 (\lambda,\nu)_{\mathrm{independent}}=(Dd/2,D+1).
\]
To compute these pairs, Proposition~\ref{prop-collapse-ideal} gives
\[
R_{\mathrm{serial}}\asymp\|q_1\|^2\|k_1\|^2,
\qquad
R_{\mathrm{independent}}\asymp\sum_i\|q_i\|^2\|k_i\|^2.
\]
The first comparison is uniform in the suffix factors. Since the analytic prior is positive, it is bounded above and below on a sufficiently small neighborhood. It therefore suffices to estimate the corresponding constant-density integrals. In polar coordinates, a single factor pair contributes, up to fixed constants,
\[
\begin{aligned}
\int_0^\epsilon\!\int_0^\epsilon
e^{-Ta^2b^2}a^{d-1}b^{d-1}\,da\,db
&=\int_0^{\epsilon^2}e^{-Tu^2}u^{d-1}
\log(\epsilon^2/u)\,du\\
&\asymp T^{-d/2}\log T,
\end{aligned}
\]
where we used $u=ab$ and then scaled $u$ by $T^{-1/2}$. In the serial network, integrating over the free suffix gives only a finite positive volume. For independent layers, the product neighborhood gives a product of $D$ such integrals. These yield the displayed pairs. The fixed constants in the loss comparison only rescale $T$, and comparison with inner and outer product neighborhoods gives the same pairs for small balls. Thus the residual ideal determines these local loss exponents throughout a parameter neighborhood, beyond the information supplied by a single Jacobian path. The exponents describe the concentration of the displayed loss integral near the fiber. To obtain a training rate, one would additionally need to analyze the optimization dynamics.

\subsection{What the contact spectrum records}\label{sec:geometry-contact-spectrum}

We finally explain the algebraic meaning of the contact orders. Let $J(t)$ be the observation Jacobian, differentiated in all native parameters and then evaluated on the path of Theorem~\ref{thm-spectrum}. A native covector is a linear functional on native parameter variations, and each row of $J(t)$ is such a covector: it gives the first-order change in one observed output. Over $\mathcal O=\mathbb R\{t\}$, quotienting all native covectors by the module generated by these observed differentials gives $\operatorname{coker}J(t)^\top$. The diagonal reduction in Appendix~\ref{appendix-d} gives
\[
\operatorname{coker}J(t)^\top
\cong\mathcal O^\zeta\oplus
\bigoplus_i\left(\mathcal O/(t^{\beta_i})\right)^{r_i(2d-r_i)}.
\]
Here $\zeta=(D-1)d^2+\sum_i r_i^2$ is the exact kernel count. A unit diagonal entry generates an entire coordinate copy of $\mathcal O$, so its quotient is zero. An entry $t^{\beta_i}$ instead gives $\mathcal O/(t^{\beta_i})$, while a zero column leaves a free copy of $\mathcal O$. Thus the $d^2$ unit factors disappear from the cokernel, and the remaining summands record
\[
 \begin{array}{c|c}
 \mathcal O^\zeta&\text{exact source-kernel count }\zeta\\
 \mathcal O/(t^{\beta_i})&\text{sensitivity order }|t|^{\beta_i}.
 \end{array}
\]
Invertible analytic coordinate changes preserve these powers. Since the coordinate changes and their inverses are bounded near zero, the same powers give the singular-value comparisons. This quotient is the module of ordinary relative differentials restricted to the path. It records first-order observation data, whereas the full relative cotangent complex also includes relations in other degrees.\footnote{Further finite-observation comparisons and polynomial examples with identical differential modules but different negative-degree cotangent cohomology are developed separately. These results concern the chosen observation algebra and are not needed for the attention theorems proved here.} The residual ideals above answer the separate question of which native parameters realize the marked function.

\section{Numerical protocols}\label{appendix-experiments}

We use two types of observations to examine the distinction between exact identification and numerical sensitivity. First, we evaluate the native Jacobian on a fixed input bank along a prescribed parameter path. We then consider recovery of one scalar factor from a noisy output mean with a known prefix. The observations are
\[
 J_{\mathcal B}(t)=D_\theta\mathcal O_{\mathcal B}|_{\theta(t)},
 \qquad Y=f(w_*^2)+\tau Z,\quad Z\sim N(0,1).
\]
Here $f$ is the scalar mean map defined in Section~\ref{one-noisy-observation-with-a-known-prefix}. The Jacobian experiments compare measured singular-value decay with the predicted contact orders, and the recovery experiment measures how that weak sensitivity affects estimation. All computations use fixed weights and inputs, without a training procedure.

\subsection{Scalar depth-four spectrum}\label{scalar-depth-four-spectrum}

We begin with the scalar experiment in Figure~\ref{fig:scalar-spectrum}. We use two tokens and four softmax layers, with all values fixed at one and $q_i=k_i=t$. For each of the following input pairs, we observe the mean of the two final output tokens: \[
(1+z_j,1-z_j),\qquad z_j=0.45+0.2j,\quad j=0,1,2,3.
\] Let $B(c)$ be the vector of these four means, with $c_i=q_ik_i$. We use Euclidean norms on both parameters and observations, and compute $D_cB$ at $t=0.08,0.04,0.02$ using 110 decimal digits. To express its sensitivity in the native query/key coordinates, we consider the two orthonormal directions in each factor pair:
\[
 (\delta q_i,\delta k_i)=\frac{(1,1)}{\sqrt2}
 \ \Longrightarrow\ \delta c_i=\sqrt2\,t,
 \qquad
 (\delta q_i,\delta k_i)=\frac{(1,-1)}{\sqrt2}
 \ \Longrightarrow\ \delta c_i=0.
\]
Consequently
\[
 \sigma_+\bigl(D_{(q,k)}B\bigr)
 =\sigma\bigl(D_cB\,(\sqrt2\,t I_4)\bigr).
\]
Thus a $4\times4$ matrix contains all positive singular values of the $4\times8$ native query/key Jacobian. The remaining four directions preserve the products $q_ik_i$ and lie in its exact kernel. In this scalar experiment, we differentiate only the query and key factors, keeping the values fixed.

The four fitted orders, computed from the final pair \(0.04,0.02\), are \[
0.999994,\qquad5.002211,\qquad17.001118,\qquad53.000562.
\] For a positive singular value, the final-pair estimator is
\[
 \widehat\beta_j=\frac{\log\sigma_j(0.04)-\log\sigma_j(0.02)}{\log(0.04/0.02)}.
\]
The measured orders agree with the predictions $1,5,17,53$ to maximum absolute discrepancy $0.002212$. In particular, at $t=0.02$ the weakest positive singular value is $3.86167\times10^{-92}$, illustrating the scale separation across layers. The saved results contain binary64 summaries of the high-precision positive values. In the figure, we plot these positive values and omit the four exact kernel directions. The dashed guides show the predicted powers anchored at the smallest scale.

\begin{figure}[htbp]
\centering
\includegraphics[width=0.49\textwidth]{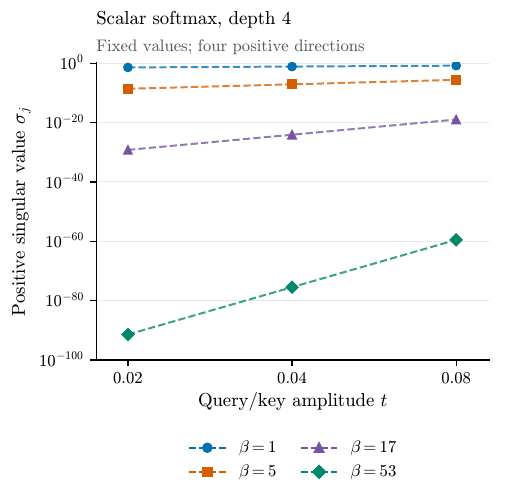}
\caption{Depth-four scalar softmax with fixed unit values. The four positive query/key singular values follow powers $1,5,17,53$. Points use saved 110-digit computations, and dashed guides show the theoretical powers anchored at the smallest scale. Four exact factor-gauge directions are omitted.}
\label{fig:scalar-spectrum}
\end{figure}

\subsection{Full native matrix spectrum}\label{full-native-matrix-spectrum}

We next differentiate all native parameters in a matrix-valued stack. We use token width $d=4$, query/key width $r=2$, depth $D=3$, and a bank of 48 four-token sequences. With a PyTorch generator seeded at 21, we first draw \(X=0.6\,\mathrm{randn}(48,4,4)\), then \(U_i=\mathrm{randn}(4,2)/2\) for all three layers, then \(W_i=U_i+0.3\,\mathrm{randn}(4,2)/2\). We evaluate the Jacobian at \[
Q_i=tU_i,\quad K_i=tW_i,\quad V_i=I,
\qquad t\in\{0.6,0.4,0.25,0.16,0.1,0.06\}.
\] We use the score convention of the theorem, with no additional $1/\sqrt r$ multiplier. All 96 query, key, and value entries are independently differentiated, and we observe all 768 scalar outputs. To compare precision levels on identical data, we import the original binary64 inputs and weights exactly into arbitrary precision and store their round-trip representations with the results.

The implementation stores tokens as rows, so its token matrix is the transpose of the manuscript convention. For this matrix $X$ with tokens as rows, let $Z=XMX^\top$ be the logits and $A$ their rowwise softmax. We propagate analytic derivatives using
\[
\begin{aligned}
\delta Z&=\delta XMX^\top+X\delta M X^\top+XM\delta X^\top,\\
\delta A_{ab}&=A_{ab}\left(\delta Z_{ab}-\sum_cA_{ac}\delta Z_{ac}\right),\\
\delta(AX)&=\delta A\,X+A\,\delta X.
\end{aligned}
\] The score-factor variations are \(\delta M_i=t(\delta Q_iW_i^\top+U_i\delta K_i^\top)\). Value-prefix variations and the output-product variation are included through the full native source differential.

The exact kernel has twelve factor directions \((\delta Q_i,\delta K_i)=(U_iB,-W_iB^\top)\), \(B\in\mathbb R^{2\times2}\), and 32 interface directions. At either internal interface \(j\), these are \[
\delta V_j=H,\quad\delta V_{j+1}=-H,\quad
\delta Q_{j+1}=-tH^\top U_{j+1},\quad
\delta K_{j+1}=-tH^\top W_{j+1},
\quad H\in\mathbb R^{4\times4}.
\] These source variations lie in $\ker J$ algebraically, without thresholding:
\[
 \dim\ker J=3(2^2)+2(4^2)=44,\qquad\operatorname{rank}J=96-44=52.
\]

To compute the 52 positive singular values, we first separate the exact kernel. For each score, we keep variations in the lower two query rows and in all key rows. Since the upper $2\times2$ block of $U_i$ is invertible, these form a complement to the factor kernel. Together with sixteen output-product directions, they give $J=F(t)T(t)$, where $F$ is $768\times52$ and $T$ is the $52\times96$ map from the original native variations to these coordinates. Although these coordinates remove the exact kernel, they need not be orthonormal in the original native metric. We therefore retain $T$ when computing singular values. To resolve the different scales, we first divide the columns of $F$ by their predicted powers and compute a QR factorization: \[
F=E\Delta,\qquad E=Q_oR,\qquad
\Delta=\operatorname{diag}(I_{16},tI_{12},t^5I_{12},t^{17}I_{12}).
\] Since $Q_o^\top Q_o=I$,
\[
 J=Q_oR\Delta T,
 \qquad \sigma_+(J)=\sigma_+(R\Delta T).
\]
We evaluate the $52\times52$ Gram matrix $(R\Delta T)(R\Delta T)^\top$ at 90 and 120 decimal digits. Figure~\ref{fig:spectra}(a) shows all 52 positive singular values at 120 digits, and the following table gives the range of fitted orders within each predicted band.

\begin{center}
\begin{tabular}{rrrr}
\toprule
Predicted order & Multiplicity & Smallest final-pair fit & Largest final-pair fit \\
\midrule
0 & 16 & 0.000019 & 0.001653 \\
1 & 12 & 0.995986 & 1.000806 \\
5 & 12 & 5.001606 & 5.009794 \\
17 & 12 & 16.994265 & 17.014769 \\
\bottomrule
\end{tabular}
\end{center}

The fits use the final pair $t=0.1,0.06$. We also test whether the weakest band corresponds to the final layer. Let $E_{<D}$ contain the divided value and earlier-layer columns, and let $E_D$ contain the divided final-layer columns. Let $P_{<D}$ be the orthogonal projection onto $\operatorname{im}E_{<D}$. Applying $I-P_{<D}$ to the final-layer columns removes the directions already represented by the stronger groups. We therefore compare
\[
 V_D=\operatorname{im}((I-P_{<D})E_D)
 \quad\text{with the weakest twelve left singular directions}.
\]
Both spaces have dimension twelve. We measure their discrepancy by the sine of their largest principal angle, which decreases from $3.01\times10^{-6}$ at $t=0.6$ to $5.21\times10^{-20}$ at $t=0.06$. Thus the weakest singular space approaches the final-layer image after removing the stronger contributions.

Increasing precision from 90 to 120 digits changes a positive singular value by at most \(1.31\times10^{-35}\) relatively. The subspace discrepancies agree through all 45 reported digits. The reconstructed native Jacobian agrees with independent float64 automatic differentiation to relative Frobenius error at most \(1.47\times10^{-16}\). These comparisons check numerical convergence and agreement of the two derivative implementations. The calculations use high-precision arithmetic without interval certification.

At $t=0.06$, the smallest positive value is $7.58144\times10^{-28}$. The declared diagnostic is
\[
 \operatorname{rank}_{\mathrm{num}}J
 =\#\{j:\sigma_j>10^{-10}\sigma_1\}=40,
 \qquad\operatorname{rank}J=52.
\]
The twelve weakest positive values lie below the chosen cutoff even in high precision, which explains the reported rank of 40. Insufficient float64 precision is a second effect: it also prevents resolution of those values. Thus the cutoff omits twelve identifiable directions in addition to the 44 exact kernel directions. We use it only as a numerical diagnostic, rather than as an exact-rank test or a threshold calibrated to observation noise.

The supplementary code includes the inputs, saved results, and generating scripts, with file descriptions in its README. The computations used mpmath 1.3.0, NumPy 2.4.3, and PyTorch 2.12.1.

\subsection{One noisy observation with a known prefix}\label{one-noisy-observation-with-a-known-prefix}

We finally consider recovery from a single noisy observation, shown in Figure~\ref{fig:spectra}(b). We use depth two and input $(1.5,0.5)$, with known first-layer factors $q_1=k_1=0.1$ and known values $v_1=v_2=1$, so $L=1$. We restrict the unknown final factors to $q_2=k_2=w\in[0,0.3]$, with true value $w_*=0.1$. Thus the final score is $c_2=w^2$, and the nonnegative balanced factor is uniquely determined by that score. This gives a one-parameter recovery problem with a fixed prefix.

The known prefix produces mean $m=1.0024999635\ldots$ and contrast $z=0.0012499661\ldots$, where contrast is half the difference between its two output tokens. Writing its two output tokens as $m\pm z$, we observe the final mean with additive Gaussian noise: \[
Y=f(w_*^2)+\tau Z,\quad Z\sim N(0,1),\qquad
f(c)=m+\frac z2\{\tanh(cz(m+z))+\tanh(cz(m-z))\}.
\] Since $m>z>0$, both arguments of $\tanh$ increase with $c$, and hence $f$ is strictly increasing on $[0,0.09]$. We can therefore recover $c$ by inversion. When noise places $Y$ outside the attainable mean interval, we first clip it to that interval and then define
\[
 Y_{\mathrm{clip}}=\min\{f(0.09),\max\{f(0),Y\}\},\qquad
 \widehat c=f^{-1}(Y_{\mathrm{clip}}),\qquad\widehat w=\sqrt{\widehat c}.
\]
We evaluate this inverse using 75 bisection steps at 100-digit precision.

The mean map satisfies $f(c)=m+bc+O(c^3)$, where $b=mz^2=1.5663213492\times10^{-6}$. We therefore scale the noise as $\tau=\alpha bw_*^2$, relative to the leading mean change at the true parameter. For the four values of $\alpha$ below, we reuse the same 400 independent standard-normal draws from NumPy seed 20261001.

For the resulting estimates $\widehat w_1,\ldots,\widehat w_{400}$, the reported summaries are
\[
 \widehat{\mathrm{MSE}}=\frac1{400}\sum_j(\widehat w_j-w_*)^2,\qquad
 \widehat{\mathrm{bias}}=\frac1{400}\sum_j(\widehat w_j-w_*),\qquad
 \mathrm{relative\ RMSE}=\frac{\sqrt{\widehat{\mathrm{MSE}}}}{w_*}.
\]

\begin{center}
\begin{tabular}{rrr}
\toprule
Noise ratio $\alpha$ & Empirical MSE in $w$ & Empirical bias in $w$ \\
\midrule
0.01 & $2.5319\times10^{-7}$ & $-4.6635\times10^{-6}$ \\
0.1 & $2.5448\times10^{-5}$ & $-1.6122\times10^{-4}$ \\
1 & $0.0029054$ & $-0.0108935$ \\
10 & $0.0172789$ & $0.0253375$ \\
\bottomrule
\end{tabular}
\end{center}

\end{document}